\documentclass[hidelinks,final]{colt2026} %
\usepackage[T1]{fontenc}    %
\usepackage{microtype}      %

\usepackage[normalem]{ulem}

\usepackage{xcolor}
\usepackage{graphicx}

\usepackage{mathrsfs}
\usepackage{mathtools,amssymb}   %
\mathtoolsset{showonlyrefs,showmanualtags}

\usepackage{bm}             %

\usepackage{booktabs}

\usepackage[boxed,linesnumbered]{algorithm2e}
\SetAlCapFnt{\small\rmfamily\centering}
\DontPrintSemicolon            %
\SetAlgoNoLine
\makeatletter
\AtBeginEnvironment{algorithm}{%
  \SetAlgoNoLine %
  \SetAlgoNoEnd
}
\makeatother

\makeatother

\usepackage{my_notation}

\usepackage{amsthm}
\usepackage{thmtools}                   %

\usepackage{my_citations}   %
\usepackage{csquotes}
\usepackage[british]{babel}     %

\usepackage{hyperref}       %
\usepackage{bookmark}       %
\usepackage{zref-clever}

\makeatletter
\let\zc@oldlabel\label
\renewcommand{\label}[1]{\zc@oldlabel{#1}\zlabel{#1}}
\makeatother

\let\cref\zcref
\zcsetup{cap,nameinlink=false}                   %

\AddToHook{env/algorithm/begin}{%
   \zcsetup{reftype=algorithm}}

\zcRefTypeSetup{algorithm}{
Name-sg = Algorithm,
name-sg = algorithm,
Name-pl = Algorithms,
name-pl = algorithms,
}

\usepackage{my_theorems}

\usepackage{enumitem}           %
\usepackage{xparse}

\title[Sharp analysis of linear ensemble sampling]{Sharp analysis of linear ensemble sampling}

\coltauthor{%
 \Name{David Janz} \Email{david.janz@stats.ox.ac.uk}\\
 \addr University of Oxford, UK
 \AND
 \Name{Arya Akhavan} \Email{arya.akhavan@stats.ox.ac.uk }\\
 \addr University of Oxford, UK
 \AND
 \Name{Csaba Szepesvári} \Email{szepesva@ualberta.ca}\\
 \addr University of Alberta, Canada%
}

\begin{document}

\maketitle

\begin{refsection}

\begin{abstract}%
  We analyse linear ensemble sampling (ES) with standard Gaussian perturbations in stochastic linear bandits. We show that for ensemble size $m=\Theta(d\log n)$, ES attains $\tilde O(d^{3/2}\sqrt n)$ high-probability regret, closing the gap to the Thompson sampling benchmark while keeping computation comparable. The proof brings a new perspective on randomized exploration in linear bandits by reducing the analysis to a time-uniform exceedance problem for $m$ independent Brownian motions. This continuous-time lens appears particularly natural here: it yields an exact representation of the relevant discrete-time processes, and we do not know another route to a sharp ES bound.
\end{abstract}

\section{Introduction}

We study ensemble sampling in the standard stochastic linear bandit problem with a fixed,
potentially infinite action set.
Linear bandits are a canonical model for studying exploration with structured reward functions:
the linear structure is simple enough to support sharp analysis, yet rich enough to expose
the statistical and computational issues that arise when observations at one action inform the
estimated rewards of others.
Ensemble sampling (ES) \citep{lu2017ensemble} is a randomized algorithm for this problem that has received considerable attention in recent years
\citep[e.g.,][]{osband2016deep,osband2018randomized,osband2019deep,lu2018efficient,qin2022analysis,dwaracherla2022ensembles,janz2023ensemble,lee2024improved,zhu2023deep,zhou2025stochastic}.
The algorithm maintains a collection (ensemble) of models, each trained on its own perturbed version of the history, and each modelling the dependence of the mean reward on actions.
In every round, ES chooses one model uniformly at random from the ensemble and selects the action that maximizes the predicted mean reward under the chosen model.
The appeal of ES is that it reduces exploration to model training and action selection from a single model, and as such it can be implemented whenever model training and single-model action selection are feasible.

ES is one of several randomized algorithms for exploration in bandit problems.
A gold standard among these is Thompson sampling (TS) \citep{thompson1933likelihood,agrawal2013thompson,abeille2017linear}, whose $n$-step regret in $d$-dimensional linear bandits is known to be $\tilde{O}(d^{3/2}\sqrt{n})$ \citep{abeille2017linear,hamidi2020frequentist}.
A long line of research has sought to establish whether ES can match this performance (see \cref{tab:results-overview}).

Despite this effort, the best results for ES fall short of achieving the same regret as TS, except in the case where the action set is finite and small enough that an ensemble size ($m$) scaling with the number of actions $k$ is tolerable.
Indeed, \citet{lee2024improved} recently showed that ES can match the regret of TS in the finite-$k$ setting; however, their result requires setting $m=\Omega(k\log(n))$, which makes the method unsuitable for large or infinite action sets.
In particular, if one insists on using this approach for the case where the action set is the unit ball in $\mathbb{R}^d$ via discretization, the discretized action set size must scale exponentially with the dimension $d$.

From a computational perspective, to keep the per-round cost comparable to that of competing methods in $d$-dimensional linear bandits, it is natural to restrict attention to the case $m=\tilde O(d)$ (where $\tilde O$ hides $\log(n)$ factors), regardless of the number of actions.
We note in passing that this scaling cannot be pushed too far: a simple lower bound (proved in this paper, see \cref{thm:lower-bound}) shows that if $m\le d/2$ then there are linear bandit instances on which ES incurs linear regret.
In the regime $m=\tilde O(d)$, the best prior regret bound for ES in linear bandits is $\tilde{O}(d^{5/2}\sqrt{n})$ \citep{janz2023ensemble}, achieved with $m=\Theta(d \log(n))$, lagging behind the regret of TS by a factor of $d$.
\emph{The main contribution here is to close this gap by showing that ES with $m=\Theta(d\log(n))$ can achieve the canonical $\tilde O(d^{3/2}\sqrt{n})$ regret rate.}

\begin{table}[tp]
\centering
\caption{Overview of results on ES for the stochastic linear bandit setting. Here, $n$ denotes the decision horizon, $d$ the dimension, and $k$ the number of arms (where a given result requires this to be finite). Thompson sampling result is per \citet{agrawal2013thompson}. For readability, regret bounds are stated omitting $\log n$ factors. Where results allow a trade-off between regret and ensemble size, we pick the ensemble size to give $\sqrt{n}$ regret. For high probability results we set $\delta = 1/n$.}
\label{tab:results-overview}
{\small
\begin{tabular}{cccc}
\toprule
Paper & Regret bound & Ensemble size & Notes \& comments
\\\midrule
\citet{lu2017ensemble} & $d^{3/2}\sqrt{n}$ & $\sqrt{k} n$ & given proof is incorrect \\
\citet{qin2022analysis} & $\sqrt{dn} $ & $dn$ & Bayesian regret, not frequentist \\
\citet{janz2023ensemble} & $d^{5/2}\sqrt{n}$ & $d\log n$ & algorithm uses symmetrisation \\
\citet{lee2024improved} & $d^{3/2}\sqrt{n}$ & $k \log n$ & $k$ can be as large as $e^d$ \\
\citet{sun2025provable} & $d^{3/2}\sqrt{n}$ & $k\log n$ & generalised linear model setting \\
this paper & $d^{3/2}\sqrt{n}$ & $\mathclap{d \log n}$ & --- \\\midrule
Thompson sampling & $d^{3/2}\sqrt{n}$ & --- & --- \\
\bottomrule
\end{tabular}
}

\end{table}

\paragraph{Technical contributions}
The primary difficulty in analysing ES---and the reason for the gap in prior results---is the dependence induced by adaptation:
although the perturbations injected into the data are independent across ensemble members and time, the fitted models are coupled by training on data collected using actions selected by the ensemble itself.

Our proof builds on \citet{janz2023ensemble}, who show that regret control reduces to establishing a high-probability lower bound on certain \emph{time-uniform self-normalized exceedance frequencies}
(i.e., the fraction of ensemble members whose self-normalized noise exceeds a fixed threshold, uniformly over time).
Our first contribution is to identify the relevant noise processes as \emph{predictable diagonal martingale transforms}:
processes of the form
\[
S_t \;=\; \sum_{s=0}^t D_s \xi_s,
\]
where $(\xi_s)_{s\ge 0}$ are the i.i.d.\@ noise variables used in model training and $D_s$ are predictable diagonal matrices (in our linear setting, $D_s$ is a scalar multiple of the identity).
A key realization here is that Gaussian perturbations unlock powerful continuous-time tools, as we describe next.

With this choice in place, we prove a representation/embedding lemma showing that such Gaussian transforms can be realized by evaluating \emph{independent} Brownian motions at random times given by the predictable quadratic variation of the transform
(key ingredients include a pinned Brownian motion/Brownian bridge construction, the Dambis--Dubins--Schwarz time change, and orthogonality/covariation arguments to obtain independence).
This step converts the coupled discrete-time noise processes arising in ES into independent Brownian motions with random clocks, isolating all adaptivity into the clocks.
An important feature of our linear ES setting is that the diagonal transforms are in fact scalar multiples of the identity, so all coordinates share a \emph{common} clock---precisely the regime in which we can control time-uniform exceedance frequencies of independent Brownian motions.
(When the clocks differ across coordinates, the problem becomes substantially more delicate.)

Our second contribution is a time-uniform lower bound on the \emph{exceedance frequency} for $m$ independent Brownian motions:
over any time window $t\in[\tau,\tau']$, the fraction of Brownian motions above the curved boundary $c\sqrt{t}$ stays above a target level $p$ uniformly in $t$ with probability at least $1-\delta$, provided $p$ is below a constant $p_0(c)$ depending on $c$ (on the order of the Gaussian tail at $c$), and with a suitable discretization step $h=h(c,p)$ (which is a constant when $c$ and $p$ are fixed; e.g.\ $c\le 1/20$ and $p\le 1/10$ allow $h=1/250$),
\[
m \;\gtrsim\; \frac{1}{p}\log\!\Big(\frac{\log(\tau'/\tau)}{h\,\delta}\Big)
\qquad
(\text{up to universal constants}).
\]
(This bound is the core probabilistic step; the final ensemble-size requirement in our bandit theorem additionally accounts for uniformity over directions.)
Key ingredients include an Ornstein--Uhlenbeck time change, boundary-crossing estimates for Brownian motion, Chernoff bounds, and a union bound over a log-time discretization.
Combining these yields the desired time-uniform self-normalized exceedance frequency guarantee, and hence our regret bound.

\paragraph{Proof insight and outlook.}
While the stochastic-analysis ingredients are classical, their role here is not: the Brownian embedding and time-change tools are what unlock a tractable analysis of the adaptive, self-normalized discrete-time exceedance problem. This continuous-time viewpoint is the key conceptual lever in the paper.
We find it striking that these tools are so effective in a purely discrete-time setting, and at present we do not know an equally sharp purely discrete alternative.

We conjecture that analogous time-uniform self-normalized exceedance frequency guarantees hold for other perturbation distributions (e.g., Rademacher), but proving this would require different tools. Plausible routes include Skorokhod embeddings and strong approximation couplings. Developing either into a sharp exceedance bound remains open.

Our use of continuous-time tools differs from prior work on discrete-time algorithms and processes, which uses continuous-time limits as approximations \citep[e.g.,][]{benveniste1990adaptive,kushner2001analysis,borkar2008stochastic,lattimore23}; here the continuous-time process gives an exact representation.
A natural direction is to use such embeddings to derive self-normalized tail inequalities for Gaussian martingale transforms, an approach that appears both in early work such as \citet{darlingerdos1956} (LIL for i.i.d. Gaussians) and more recently in \citet{lu2022em} (self-normalized bounds for scalar Gaussian martingale transforms in SDE discretization analysis).
Given the elegance of the approach, it seems likely that it could find further applications in deriving sharp inequalities for discrete-time stochastic processes.

Finally, it is an interesting open direction to understand whether related mechanisms persist in nonlinear settings when the effective noise retains an approximately conditionally independent component structure.

\section{Notation and problem setting}
\label{sec:prelims}
The purpose of this section is to introduce our notation and the linear bandit setting we consider.
We start with notation that will be used throughout. We let $\Np$ denote the set of positive integers, for $n\in \Np$ we let $[n]=\{1,\dots,n\}$. We use $\R$ to denote the set of real numbers,
$\Rd$ to denote the $d$-dimensional Euclidean space and $\Bd$ to denote the unit ball of this space under the Euclidean norm $\norm{\cdot}:=\norm{\cdot}_2$, while
$\Sd$ denotes the corresponding sphere.
For $x,y\in \Rd$, $\ip{x,y}=x^\top y$ denotes their inner product.
For a positive definite matrix $A \in \R^{d \times d}$, we let $\norm{x}_A = \sqrt{x\tran A x}$ denote the associated weighted 2-norm.
For a probability distribution $P$, we let $P^{\otimes m}$ denote the $m$-fold product distribution of $P$. We use $\ind{\cdot}$ to denote the indicator that returns one when its argument is true and zero otherwise.

\paragraph{Linear bandits}
The standard subgaussian stochastic linear bandit setting
\citep[Chapter 19]{lattimore2020bandit}
is as follows:
Fix $d \in \Np$, a closed action set $\cX \subset \Bd$, and a fixed parameter $\theta_\star \in \Bd$. In a stochastic linear bandit environment $\cE$ specified by $(\cX,\theta_\star)$, a learner interacts with the environment in rounds. The action set $\cX$ is known to the learner, but the parameter $\theta_\star$ is unknown. The learner sequentially selects actions from $\cX$, sends them to the environment, and the environment responds with noisy rewards that follow a linear model: in round $t\in \Np$ the learner selects action $X_t \in \cX$, then it receives reward $Y_t \in \R$ such that the mean reward is $\ip{X_t,\theta_\star}$.
We assume that the environment generates noise that has light tails. In particular,
the noise $\eta_t = Y_t - \ip{X_t,\theta_\star}$ is assumed to be conditionally $1$-subgaussian, given the past actions and rewards and randomizations of the learner. That is, we assume that for all $t \in \Np$,
\[
  \E\big[\exp(s \eta_t) \mid \cG_t\big] \leq \exp\left(\frac{s^2}{2}\right)\,, \quad \text{for all $s \in \R$}\,,
\]
where $\cG_t$ is a $\sigma$-algebra that makes $X_1, Y_1, \dots, X_{t-1}, Y_{t-1}, X_t$, as well as any other random variables the learner used to generate $X_1,\dots,X_t$, measurable.

We fix a horizon $n \geq 1$ and measure the performance of the learner by the $n$-step (pseudo)regret, which captures the shortfall of mean rewards generated by the learner compared to that of an oracle that always selects the optimal action:
\[
  R_n = \max_{x_\star \in \cX}\sum_{t=1}^n \langle x_\star - X_t, \theta_\star \rangle\,.
\]
Our goal will be to show that when $(X_t)$ is selected by ensemble sampling, which we specify next, $R_n$ is small with high probability.

\section{Ensemble sampling}
In this section we introduce ensemble sampling (ES) and recall the state-of-the-art result available for it in the linear bandit setting.
\paragraph{Algorithm description}
To estimate the unknown parameter $\theta_\star$, ES uses ridge regression estimates based on the data collected so far with some perturbations. Recall that ridge regression estimates the parameter $\theta_\star$ by minimising a regularised squared loss over the data collected so far. Since the regulariser is chosen to be the squared Euclidean norm, the resulting estimates and confidence ellipsoids enjoy closed-form expressions. In particular,
given observations $(X_1, Y_1, \dots, X_{t-1}, Y_{t-1})$ up to time step $t-1$ and a regularisation parameter $\lambda > 0$, the $\lambda$-regularised ridge regression estimate for $\theta_\star$ is given by
\begin{equation}\label{eq:rr_closed_form}
  \hat\theta_{t-1} = V_{t-1}^{-1} \sum_{i=1}^{t-1} X_i Y_i \spaced{where} V_{t-1} = V_0 + \sum_{i=1}^{t-1} X_i X_i\tran \spaced{and} V_0 = \lambda I\,.
\end{equation}
Now, ES will maintain an ensemble of $m$ such ridge regression estimates $(\theta_{t}^1, \dots, \theta_{t}^m)_{t\ge 0}$. The estimates for model $j\in [m]$ are obtained by ridge regression, but with two changes: First, the regularisation is not towards zero, but towards a random vector
$\zeta^j\in \Rd$ drawn from some fixed distribution, before seeing any data (``the prior'').
Second, the data used for training is perturbed by adding noise to the rewards, as the rewards are observed. The noise is drawn afresh for each model and each time step, but once chosen, it remains fixed for that model and time step. In particular, if these noise variables are $(\xi_t^j \colon t \geq 1, j \in [m])$ then the $j$th model parameters in the ensemble will be given by the ridge regression estimate
\begin{align}
  \theta_{t-1}^j = V_{t-1}^{-1} \left(\sum_{i=1}^{t-1} X_i (Y_i + \xi_i^j) + \sqrt{\lambda} \zeta^j \right)
  = \hat\theta_{t-1} + V_{t-1}^{-1} S_{t-1}^j
  \,, \quad j \in [m]\,,
  \label{eq:fixednoise}
\end{align}
where
\[
S_{t-1}^j = \sqrt{\lambda} \zeta^j + \sum_{s=1}^{t-1} X_s \xi^j_s\,, \quad j \in [m]\,.
\]
Action selection in round $t\in [n]$ proceeds by taking a random index $J_t$ uniformly on $[m]$,
setting the current model parameter to $\theta_t = \theta_{t-1}^{J_t}$ and selecting
\[
  X_t \in \argmax_{x \in \cX} \langle x, \theta_{t}\rangle\,.
\]

\paragraph{Noise scale} It is standard to scale the noise injected into the data to ensure sufficient exploration.
The idea of inflating the noise can be traced back to \citet{agrawal2013thompson},
who, in the context of Thompson sampling for linear bandits, proposed inflating the posterior variance by a $\sqrt{d \log(n)}$ factor.
Here, following \citet{janz2023ensemble}, we consider a refinement where we allow noise scaling to be data-dependent.
This has the potential to decrease the regret in benign cases,
while keeping the computation comparable -- provided that the computation of the data-dependent scaling term is cheap (which it is in the case considered below).
The data-dependent scaling will also clarify how the noise level relates to controlling exploration.
The specific modified formula for the model parameters we will use is
\begin{align}
  \theta^j_{t-1} = \hat\theta_{t-1} + \bar{\gamma} \beta_{t-1}^\delta V_{t-1}^{-1} S_{t-1}^j\,, \quad j \in [m]\,
\end{align}
where $\beta_{t-1}^\delta$ is the data-dependent scaling parameter and $\bar\gamma$ is a tuning parameter
(and the superscript is used to denote the dependence on this $\delta$, not exponentiation).
The reason for writing the scale as a product of these two terms will be explained momentarily.
This completes the description of ES. The full algorithm is given in \cref{alg:es-lin},
where we leave the choice of the prior $P_0$ and the noise distribution $Q$ as inputs.
Common sense suggests that these should be centered, light-tailed, but nontrivial (positive variance) distributions.
\DontPrintSemicolon
\begin{algorithm}[tb]
\caption{Linear ensemble sampling with adaptive noise levels}\label{alg:es-lin}
\textbf{inputs}:
ensemble size $m \in \Np$,
inflation parameter $\bar{\gamma} > 0$,
regularisation parameter $\lambda > 0$,
confidence parameter $\delta \in (0,1)$,
distribution $P_0$ over $\R^d$ and $Q$ over $\R$\;
\vspace{\baselineskip}
initialise the design matrix $V_0 = \lambda I$ and the accumulator $S_0 = 0$\;
let $(\zeta^1,\dots,\zeta^m)\sim P_0^{\otimes m}$,  $S_0^j=\sqrt{\lambda}\zeta^j$ and
$\theta^j_0 = \bar{\gamma} \beta_0^\delta V_0^{-1} S^j_0$
for each $j \in [m]$\;
\vspace{\baselineskip}
\For{$t=1,2,3,\dots$}{
  select an ensemble index $J_t \sim \operatorname{Unif}([m])$ and set $\theta_t = \theta_{t-1}^{J_t}$\;
  compute action $X_t \in \argmax_{x\in \cX} \ip{x, \theta_t}$\;
  send $X_t$ to the environment and receive feedback $Y_t\in \R$\;
  compute the new design matrix $V_t = V_{t-1} + X_t X_t\tran$ and let
  \[
    S_t = S_{t-1} + Y_t X_t \spaced{and} \htheta_{t} = V_t^{-1} S_t
  \]\; \vspace{-0.2in}
  compute $\beta_t^\delta$ as in \eqref{eq:beta},
  sample $(\xi_t^1, \dots, \xi_t^m)$ from $Q^{\otimes m}$ and let
  \[
    S_t^j = S_{t-1}^j + \xi_t^j X_t \spaced{and}
    \theta_{t}^j = \htheta_{t} + \bar{\gamma} \beta_{t}^\delta V_t^{-1} S^j_t\,.
  \]
}

\end{algorithm}

Let us now return to clarifying how the scaling of the noise helps control exploration.
The idea is to match the noise level to the size of a standard confidence ellipsoid for $\theta_\star$ up to the constant multiplier $\bar\gamma$.
This is what in turn allows the algorithm to explore sufficiently: By choosing this noise scaling, the perturbed model parameters will hopefully spread out in the confidence ellipsoid, making it likely that some fraction of the models in the ensemble are optimistic (i.e., yield a larger optimal reward than the true parameter). If this fraction can be kept above a positive constant uniformly over time, then the regret can be controlled (see \citet{janz2023ensemble} for details).
Once the noise level is fixed, the central question becomes how large $m$ must be in $d$ dimensions to ensure a constant fraction of optimistic draws.
\cref{fig:optimism-geometry} in \cref{app:optimism-figure-section} illustrates the underlying geometry.

The standard radius mentioned above is given by
\begin{equation}\label{eq:beta}
    \beta_{t-1}^\delta = \sqrt{\lambda} + \sqrt{2 \log (1/\delta) + \log (\det (V_{t-1}) / \lambda^d)}\,, \quad t \in \Np\,,
\end{equation}
where $\delta \in (0,1)$ is a confidence parameter.
In particular, owing to the assumption that the reward noise is subgaussian, the following holds:
\begin{lemma}[Confidence ellipsoids, \citep{abbasi2011improved}]
    \label{lem:conf_ellip}
    With probability at least $1-\delta$, $\theta_\star\in \bigcap_{t\in \Np} \Theta_{t-1}^\delta$, where
    \[
        \Theta_{t-1}^\delta = \{\theta \in \Rd \colon \norm{\theta-\hat\theta_{t-1}}_{V_{t-1}} \leq \beta_{t-1}^\delta\}, \quad t \in \Np\,.
    \]
\end{lemma}
As to the magnitude of $\beta_{t-1}^\delta$,
by the elliptical potential lemma (e.g., Lemma 19.4 of \citet{lattimore2020bandit}),
it holds that
\begin{equation}\label{eq:beta_ub}
\beta_{t-1}^\delta \le \sqrt{\lambda} + \sqrt{2 \log (1/\delta) + d \log (1 + n/(\lambda d))}\,.
\end{equation}
The noise perturbation magnitude can be set to the upper bound rather than tracking $\beta_{t-1}^\delta$ itself, yielding a fixed noise level, but perhaps at some (small) loss of exploration efficiency.

\paragraph{Computational complexity.}
With an incremental implementation,
the complexity of a single step of \cref{alg:es-lin}
is on the order of $d^2 + m d^{\omega-1}$ per round,
where $\omega \ge 2$ denotes matrix multiplication complexity (so $d^\omega$ is the cost of multiplying $d \times d$ matrices). The $d^2$ term comes from updating the design matrix $V_t$ and computing its inverse using the Sherman-Morrison formula. The $m d^{\omega-1}$ term comes from packing $S_t^1, \dots, S_t^m$ into a $d \times m$ matrix and then multiplying it with the $d \times d$ matrix $V_t^{-1}$ from the left to get all the model parameters $\theta_t^1, \dots, \theta_t^m$.
As noted earlier, avoiding linear regret in the worst case requires $m\ge d/2$ (\cref{thm:lower-bound}). Thus, absent further algorithmic innovations, sublinear-regret ES has an inherent per-round overhead relative to TS even before the $\log n$ factor ($d^\omega$ vs.\ $d^2$).

\paragraph{Current state of the art.}
The state-of-the-art regret bound for ES in linear bandits is due to \citet{janz2023ensemble}.
However, this result is stated for a slightly different version of \cref{alg:es-lin}, when $\theta_t$ is defined as
$\theta_t = \htheta_{t-1} + \bar{\gamma} \beta_{t-1}^\delta \xi_t V_{t-1}^{-1} S_{t-1}^{J_t}$, where
$\xi_t$ is a Rademacher random variable, which is chosen at the same time as $J_t$, independently of $J_t$.
While \citet{janz2023ensemble} needed this symmetrization to carry out their analysis,
it is not necessary for our analysis, and we work with \cref{alg:es-lin} as stated
(yet, the symmetrization does not hurt either and may even help in practice). Their result, Theorem 1 in \citet{janz2023ensemble}, gives that if we take $\lambda=5$ and $m$ to be as small as their constraints allow, then, up to logarithmic-in-$n$ factors, the regret bound scales as $d^{5/2}\sqrt{n}$ and  the ensemble size as~$d \log(n)$.

\section{Results}

Our main results are as follows:
\begin{theorem}[restate = theoremLinBound, name = {Regret bound for \cref{alg:es-lin}, linear ensemble sampling}]\label{thm:regret-bound}
Fix $n \geq \max\{2,d\}$, $\delta\in(0,1/4)$, $\lambda \geq 80$, $m \geq 2000d\log(n/\delta)$, and $\bar{\gamma} = 40$.
Consider the $n$-round interaction between \cref{alg:es-lin}
when $P_0$ and $Q$ are chosen to be standard normal, and
a stochastic linear bandit environment $\cE$ specified by a closed action set $\cX \subset \Bd$ and an unknown parameter $\theta_\star \in \Bd$.
Then, with probability at
least $1-4\delta$, the $n$-step regret $R_n$ enjoys the bound
\begin{equation}
      R_n \le
        C \beta_{n-1}^{\delta} \sqrt{ d \log(n/\delta)}
        \left(\sqrt{dn \log\left(1 + \frac{n}{d\lambda}\right)} + \sqrt{\left(\frac{n}{\lambda} + 1\right)\log\left(\frac{\sqrt{n/\lambda+1}}{\delta}\right)}\right) \,,
\end{equation}
where $C>0$ is a universal constant.
\end{theorem}
\begin{theorem}[name=,restate=lowerBoundThm]\label{thm:lower-bound}
  Take $\cX=\Bd$. Fix any law $P_0$ on $\Rd$ and let $m\le d/2$. Suppose that ensemble sampling is initialised with $(\zeta^1, \dots, \zeta^m) \sim P_0^{\otimes m}$, any choice of $\bar\gamma,\lambda > 0$, and the tie-breaking rule that for every $t \geq 1$, $\theta_t = 0 \implies X_t = 0$. Then there exists a $\theta_\star \in \Sd$ depending on $P_0$ only, such that
  \[
    \P{\forall n\geq 1, R_n \ge n/4} \geq 1/2.
  \]
\end{theorem}
The second result, proved in \cref{sec:ensemble-size-lb}, shows that the ensemble size requirement in \cref{thm:regret-bound} is close to its minimum.
The first result shows that if we take $\lambda=80$ (say) and take $m$ to be as small as possible given the constraint of \cref{thm:regret-bound}, then the regret bound scales as $d^{3/2}\sqrt{n}$ up to logarithmic-in-$n$ factors, while the ensemble size scales as $d \log(n)$.
This closes the gap to the regret of Thompson sampling in linear bandits, while keeping the ensemble size scaling only linearly with the dimension $d$ and logarithmically with the horizon~$n$.

A few remarks are in order:
First, there is a trade-off between the ensemble size $m$ and the noise level (governed by $\beta_{t-1}^\delta\bar{\gamma}$ here).
In particular, by increasing $m$ while keeping the noise level fixed, the effective noise level increases.
Since increasing $m$ incurs computational cost, while increasing the noise level does not, for the purposes of increasing the effective noise level, it is better to increase the noise level parameters while keeping $m$ as small as possible.

Second, \citet{janz2023ensemble} mentions that if one changes the algorithm to use max-over-ensemble action selection (i.e., selecting the action that maximizes the predicted reward over all models in the ensemble, instead of selecting the action according to a randomly selected model), then their proof techniques would yield a regret bound scaling as $d^{3/2} \sqrt{n}$, while keeping the ensemble size scaling with $d \log(n)$. The downside of max-over-ensemble action selection is that it is more expensive as it requires solving $m$ linear maximization problems over the action set in each round, instead of just one. When this maximization is expensive (certain polytope action sets), it is better to use the random model selection approach of \cref{alg:es-lin}. Our results imply that with random model selection, one can also achieve the $d^{3/2} \sqrt{n}$ regret, while keeping the ensemble size scaling as $d \log(n)$. In our case, max-over-ensemble action selection would not yield further benefits: the ``bonus'' that we can achieve with this scales at best with $\sqrt{\log(m)}$, which is too small to be effective for the ensemble size regime we care about.

\section{Proof of \cref{thm:regret-bound}}
When we omit details below, full proofs are deferred to the appendix. We start with a lemma extracted from the approach of \citet{janz2023ensemble}.
For the proof, see \cref{app:regret-via-exceedance}.
\begin{lemma}[name=Self-normalized exceedance frequencies control regret,restate=exceedanceLemma]\label{lem:regret-via-exceedance}
Fix $p,\delta\in (0,1)$, $m,n \in \Np$, $\bar\gamma,\lambda\geq 1$.
  Assume that with probability $1-\delta$, it holds that
  \[
  \min_{t\in [n]}\inf_{u\in\Sd} E_{t,m}(u,1/\bar\gamma) \ge p,
  \]
  where for $u\ne 0$, $c>0$,
  \[
  E_{t,m}(u,c) = \frac{1}{m}\sum_{j=1}^m
  \ind*{\ip{u,S_{t-1}^j}/\norm{u}_{V_{t-1}} \ge
  c
  }.
  \]
For $\alpha \in (0,1)$, let \[
  \gamma_{t}^{\alpha} = \sqrt d + 2\sqrt{2\log(4/\alpha) + d\log(1+t/(\lambda d))}\,.
\]
Then, with probability $1-4\delta$, for all $t\in [n]$,
  \[
  R_t \;\le\; \frac{2\bar\gamma}{p} \gamma^{\delta/n}_{t-1} \beta_{t-1}^\delta
  \left(2\sqrt{2d t \log\left(1 + \frac{t}{d\lambda}\right)} + \sqrt{2(4t/\lambda + 1)\log\left(\frac{\sqrt{4t/\lambda+1}}{\delta}\right)}\right)\,.
  \]
\end{lemma}

The term $\smash{\gamma_{t-1}^\alpha}$ is such that for any $j\in [m]$ with probability at least $1-\alpha/2$, for all $t\in [n]$, $\smash{\norm{ V_{t-1}^{-1/2}S_{t-1}^j} \le \gamma^\alpha_{t-1}}$. The bound uses that $S_{t-1}^j$ is a sum of the Gaussian prior $\sqrt{\lambda}\zeta^j$, which is handled using a standard Gaussian concentration bound, and the noise terms $\sum_{s=1}^{t-1} X_s \xi_s^j$, which are bounded using self-normalized martingales \citep{abbasi2011improved}. In the above, we use $\gamma^\alpha_{t-1}$ with $\alpha = \delta/n$, such that we may take a union bound over the $n$-steps; while the bound is time-uniform for a fixed $j$, we need it to hold for the particle selected in each round $t \in [n]$. We could union-bound over the $m$ particles instead, but then regret would grow with $m$, which would be unnatural; larger ensembles should be better.

From this result, it is clear that the proof will be done if we establish that for our `not too large' choice of $m$, the exceedance frequencies $E_{t,m}(u,c)$
for the self-normalized martingale transforms $\ip{u,S_{t-1}^j}/\norm{u}_{V_{t-1}}$
are lower bounded by a constant uniformly over $u\in\Sd$ and $t\in [n]$ with high probability.
The rest of the proof is devoted to establishing this.

First, a covering argument can be used to show that it suffices to control the exceedance frequencies at a fixed direction $u\in \Sd$.
To state a clean result, let $\cE_P$ denote the event
that for all $t\in [n]$ and $j\in [m]$, $\norm{ V_{t-1}^{-1/2}S_{t-1}^j} \le \gamma^{\delta/m}_{t-1}$ (here, we do use $\alpha = \delta/m$ for a union bound over the $m$ particles; it makes the analysis easier, and the resulting upper bound does not feature in the regret bound).
A similar argument was used by \citet{janz2023ensemble}, but there only the process extrema are controlled, not exceedance fractions.

\begin{lemma}[
  restate=fromFixedToUniform, name={From fixed to uniform direction}]\label{lem:from-fixed-to-uniform}
Assume that there exists $p_0:(0,\infty) \to (0,1)$ such that
for any fixed $u\in \Sd$, $c>0$, $0<p<p_0(c)$, $0<\delta<1$,
it holds that
if $m\ge m_0(c,\delta,p)$ then with probability $1-\delta$,
for all $t\in [n]$, $E_{t,m}(u,c) \ge p$.
Now, let
\[\bar\gamma>0\,, \quad \delta \in (0,1)\,, \quad p\in (0,p_0(2/\bar\gamma))\,, \quad L=2 \gamma_{n-1}^{\delta/m} \sqrt{1+n/\lambda} \,, \quad  \epsilon=1/(\bar\gamma L)\,.
\]
Then, if $m\ge m_0(2/\bar\gamma, (\delta/2)/(3/\epsilon)^d,p)$,
then with probability at least $1-\delta$, for all $u\in\Sd$ and all $t\in [n]$, $E_{t,m}(u,1/\bar\gamma) \ge p$.\end{lemma}
\begin{proof}
  The result follows from showing that the maps $u \mapsto \ip{u,S_{t-1}^j}/\norm{u}_{V_{t-1}}$ are $L$-Lipschitz for all $j\in [m]$ and all $t\in [n]$
  on an event of probability at least $1-\delta/2$, which we do in \cref{appendix:lipschitz-proof}.
  Then, a standard covering using an $\epsilon$-net of size $(3/\epsilon)^d$
  can be used to extend the fixed direction result to all directions.
\end{proof}

By the previous result, it suffices to control the exceedance frequencies at a fixed direction $u\in \Sd$.
Hence, in what follows, we fix $u\in \Sd$ and hide the dependence on it.
We will find it useful to introduce the notation
\[
M_{t} = (\ip{u,S_{t}^j})_{j=1}^m\,, \qquad t=0,1,\dots,n-1\,.
\]
(Deviating from our earlier notation, we will use $M_{t,j}$ to denote the $j$th component of the vector $M_t$ as this allows us to use linear algebra notation more easily.)
Note that $(M_t)_{0\le t \le n-1}$ is an $\R^m$-valued, zero-mean martingale. In what follows, we study the self-normalized exceedance frequencies induced by this martingale.

Recalling the definitions of $S_{t}^j$, we see that $M_t$ is obtained from ``mixing'' $m$ independent standard Gaussian noise processes:
\[
M_t = \sqrt{\lambda}\xi_0 + \sum_{s=1}^{t} \ip{u,X_s} \xi_s\,,
\]
where, by slightly abusing notation,
we denoted $\xi_0 = (\ip{u,\zeta^j})_{j=1}^m$ and $\xi_s = (\xi_s^j)_{j=1}^m$ for $s\ge 1$.

Notably, while $X_s$ may depend on $\xi_0,\dots,\xi_{s-1}$, it is independent of $\xi_s$.
Importantly, the mixing is \emph{diagonal} in the sense that each component of the martingale is obtained by mixing only the corresponding component of the noise processes.
In particular, in $M_t$, the variance of $\xi_s^j$ is scaled by $\ip{u,X_s}^2$, but no ``rotational'' dependence is introduced between $(\xi_s^j)_{j=1}^m$ across $j$.

Now, if we recall that at any fixed time $t>0$, a standard Brownian motion gives rise to a standard normal random variable with variance equal to the time parameter $t$,
we may view the process $(M_t)_{t\ge 1}$ as obtained by running $m$ independent Brownian motions and reading out their values, sequentially,
at the random times $\norm{u}_{V_{t}}^2 = \lambda\norm{u}^2 + \sum_{s=1}^{t} \ip{u,X_s}^2$.
Indeed, imagine that we start by running $m$ independent standard Brownian motions for a period of length $t_0=\lambda$. This gives rise to the prior term $M_0=\sqrt{\lambda}\ip{u,\zeta^j}$.
Next, we \emph{continue} each Brownian motion, independently of each other for a period of length $\ip{u,X_1}^2$. Reading out the values, we obtain $M_1$. Continuing, we eventually obtain all of $M_0,M_1,\dots,M_{t}$.
The gain here is that now the problem of controlling the exceedance frequencies of the self-normalized martingale transforms is reduced to controlling the exceedance frequencies of $m$ independent Brownian motions observed at random times. This is similar to switching from the ``sequential'' to the stacked reward model in bandits \citep{lattimore2020bandit}.

The astute reader will note that the Brownian motions are \emph{continued} for each period, not restarted. Therefore, whether they are independent of each other, and whether they are even standard Brownian motions, requires a justification. Next, we show that indeed they are independent standard Brownian motions.
We establish this for general diagonal martingale transforms of Gaussian noise, slightly generalizing our setting.
If you have not seen the continuous-time martingale notation in a while, there is a short primer in \cref{sec:cont-time-prelims}.

\paragraph{An embedding result for diagonal martingale transforms of Gaussian noise}
Fix $m\ge 1$. We say that an $\R^m$-valued process
\[
M_t = \sum_{s=0}^t U_s \xi_s
\]
is a \emph{diagonal martingale transform of a standard $m$-dimensional Gaussian noise} if
there exists a filtration $(\mathcal F_t)_{t\ge -1}$ such that the following hold:
\begin{itemize}
\item $(\xi_s)_{s\ge 0}$ are i.i.d.\ $\mathcal N(0,I_m)$ and $\xi_s$ is independent of $\mathcal F_{s-1}$ for each $s\ge 0$,
\item $U_s\in \mathbb{R}^{m\times m}$ is $\mathcal F_{s-1}$-measurable and diagonal for each $s\ge 0$.
\end{itemize}
(In our case, we can take $\mathcal F_{-1}$ to be the trivial $\sigma$-algebra, $U_0 = \sqrt{\lambda} I_m$, and for $s\ge 1$,
$U_s=\ip{u,X_s}I_m$ with $\mathcal F_{s-1}$ the pre-perturbation history up to the choice of $X_s$.)

\begin{theorem}[
  restate = theoremDiagGaussEmbedding, name = {Diagonal Gaussian embedding}]\label{thm:diag-gauss-embed}
Fix a horizon $n\in\N$. Assume $M_t = \sum_{s=0}^t \diag(D_s) \xi_s$ is a diagonal martingale transform of a
standard $m$-dimensional Gaussian noise, for $0\le t \le n-1$.
Define the coordinate clocks
\[
A_{t,j}^2 \;:=\; \sum_{s=0}^t D_{s,j}^2,\qquad j\in[m].
\]
Then, on an extension of the probability space, there exist independent standard Brownian motions
$W^1,\dots,W^m$ such that for all $t\in\{0,1,\dots,n-1\}$ and all $j\in[m]$,
\[
M_{t,j} \;=\;  W^j\!\big(A_{t,j}^2\big).
\]
\end{theorem}
\begin{proof}[Proof sketch]
The proof is based on the classical Dambis--Dubins--Schwarz (DDS) representation theorem
(\cref{thm:DDS}) for continuous-time martingales.
As explained above, the idea is to first embed each component $M_{t,j}$ into a continuous-time martingale using pinned Brownian segments (see \cref{lem:pinned-bm-seg} below) and then apply the DDS theorem to obtain the desired Brownian motions. Then, Knight's construction
(\cref{thm:knight})
can be used to ensure the independence of the Brownian motions.
\end{proof}

With \cref{thm:diag-gauss-embed} in hand, we can write the exceedance frequencies $E_{t,m}(u,c)$ as
\[
E_{t,m}(u,c) = \frac{1}{m}\sum_{j=1}^m
  \ind*{ W^j\big(\norm{u}_{V_{t-1}}^2\big)/\norm{u}_{V_{t-1}} \ge
  c
  }.
\]
Noting that $\norm{u}_{V_{t-1}}^2 \in
[\lambda,\lambda + n-1]
\subset
[\lambda,\lambda + n]
$ for all $t\in [n]$,
we then have
\begin{align}
  \min_{t\in [n]} E_{t,m}(u,c) \geq
  \inf_{t \in [\lambda, \lambda + n]} \frac{1}{m}\sum_{j=1}^m
    \ind*{ W^j(t)/\sqrt{t} \ge
    c
    }.    \label{eq:exbmlb}
\end{align}
The final step is to control the expression on the right-hand side.

\paragraph{Exceedance frequencies of Brownian motions}
For this, we present the following result:
\begin{theorem}[
  restate = theoremExcBrownian,
  name = {Time-uniform exceedance count for Brownian motions}]
  \label{thm:bm-quantile-K}
Let $W^1,\dots,W^m$ be independent standard Brownian motions. Fix $0<\tau < \tau'<\infty$, positive constants $c>0$ and $0<p<p_0(c):=\tfrac14(1-\Phi(c))$.
Then there exists $0<h\le 1$ such that, for any $\delta\in(0,1)$, if
\begin{align}
  m \ge \frac{4}{p}\log\!\Big(\frac{\ceil{\log(\tau'/\tau)/h}}{\delta}\Big),
  \label{eq:mBMlb}
\end{align}
then with probability at least $1-\delta$,
\[
  \inf_{t \in [\tau, \tau']} \frac{1}{m}\sum_{j=1}^m
  \ind*{ W^j(t)/\sqrt{t} \ge
  c
  } \ge p\,.
\]
Furthermore, define
\[
\varepsilon \;:=\; \frac{\Phi^{-1}(1-4p)-c}{3}
\qquad\text{and}\qquad
h_{\star} \;:=\; \min\!\Big\{1,\; 2\log\!\Big(\frac{c+3\varepsilon}{c+2\varepsilon}\Big),\; \log\!\Big(1+\frac{2\varepsilon^2}{\log 8}\Big)\Big\}.
\]
Then the conclusion holds for any choice of $h\in(0,h_{\star}]$.
In particular, if $c\le 1/20$ and $p\le 1/10$, then the choice $h=1/250$ is admissible.
\end{theorem}

\begin{proof}[Proof sketch]
We start by creating a geometric grid of points in $[\tau,\tau']$ with ratio $e^h$.
At each grid point we lower bound the exceedance count using a binomial Chernoff bound.
The Markov property of the Ornstein--Uhlenbeck process plus a Brownian supremum estimate then show that a large count at the grid point persists throughout the next interval.
Finally, a union bound over all grid intervals yields the time-uniform conclusion.
\end{proof}

\paragraph{Putting things together}
It remains to put the pieces together to prove \cref{thm:regret-bound}.
We start by establishing the fixed-direction exceedance control premise needed for \cref{lem:from-fixed-to-uniform}:
\begin{proposition}[Fixed-direction exceedance control premise]
  \label{prop:fixed_dir_premise}
A choice that makes the premise of \cref{lem:from-fixed-to-uniform} hold (in the regime $c\le 1/20$ and $p\le 1/10$) is to take $p_0$ as in \cref{thm:bm-quantile-K} and $m_0$ to be
\[
m_0(c,\delta,p) := \frac{4}{p}\log\!\Big(\frac{\lceil 250 \log((\lambda+n)/\lambda) \rceil}{\delta}\Big)\,.
\]
\end{proposition}
\begin{proof}
  We want to show that
  for any fixed $u\in \Sd$, $c>0$,
  $0<p<p_0(c)$,
  and $0<\delta<1$, it holds that
  if $m\ge m_0(c,\delta,p)$ then with probability $1-\delta$,
  for all $t\in [n]$, $E_{t,m}(u,c) \ge p$.
  By \eqref{eq:exbmlb}, it suffices to control
  $\inf_{t \in [\lambda, \lambda + n]} \frac{1}{m}\sum_{j=1}^m
    \ind*{ W^j(t)/\sqrt{t} \ge
    c
    }$ from below.
To do this, we use \cref{thm:bm-quantile-K}, the ``furthermore'' part,
  setting $\tau=\lambda$ and $\tau'=\lambda+n$ and choosing $h=1/250$ (which is admissible).
  The result is obtained by plugging into \eqref{eq:mBMlb}.
\end{proof}

Having established the premise of \cref{lem:from-fixed-to-uniform}, we can now use it to obtain uniform exceedance control. From this we get the following result:

\begin{corollary}[
  restate = corUniformExceedance, name = {Uniform exceedance control for ES}
]\label{cor:uniform-exceedance}
Assume $\lambda\ge 1$, $\bar\gamma=40$, $n\ge \max(2,d)$, $\delta\in(0,1/2)$. Write $\ell := \max\{1,\log (n/\delta)\}$. If $m\ge 2000 d\ell$, then with probability at least $1-\delta$,
\[
\min_{t\in[n]}\ \inf_{u\in\mathbb S^{d-1}_2}\ E_{t,m}\!\Big(u,\frac{1}{\bar\gamma}\Big)\ \ge\ \frac{1}{10}.
\]
\end{corollary}
\begin{proof}[Proof sketch]
We use \cref{prop:fixed_dir_premise} to verify the premise of \cref{lem:from-fixed-to-uniform}
with $p=1/10$ and $c=1/20$. Lemma~\ref{lem:from-fixed-to-uniform} then gives the desired
uniform conclusion, but requires checking its condition on $m$, which is expressed via
the threshold $m_0(\cdot)$ from Proposition~\ref{prop:fixed_dir_premise}. Evaluating this
condition leads to a log-inequality in $m$; solving it gives the sufficient condition
$m \ge 2000 d \ell$. See \cref{app:cor-uniform-exceedance}.
\end{proof}

\begin{proof}[Proof of \cref{thm:regret-bound}]
Combining \cref{lem:regret-via-exceedance} and \cref{cor:uniform-exceedance} reduces the proof to bounding
$\gamma_{t-1}^{\delta/n}$, which is carried out in \cref{app:gamma-bound} (see \cref{lem:gamma-bound}).
\end{proof}
The dependence on $1/p$ in \cref{lem:regret-via-exceedance} is unavoidable: even at a fixed
time the exceedance count is binomial with mean $p$.
For readers curious about the exceedance problem itself, \cref{app:diag-gauss-exceedance} offers a short discussion
of diagonal Gaussian martingale transforms that clarifies why the scalar-clock case is tractable and why the general
diagonal case is genuinely harder.

\section{Conclusions}
We provided a sharp frequentist analysis of Gaussian linear ensemble sampling.
Our main result, \cref{thm:regret-bound}, shows that with an ensemble size scaling linearly in the
dimension and logarithmically with the time horizon, ES achieves the canonical $\tilde O(d^{3/2}\sqrt{n})$-type regret rate
with high probability in the standard stochastic linear bandit setting.

Technically, the proof reduces the core exploration question in ES to a time-uniform exceedance problem for
self-normalized diagonal Gaussian martingale transforms.
The key representation step, \cref{thm:diag-gauss-embed}, realizes such transforms as independent Brownian motions
evaluated at predictable quadratic-variation clocks, which enables us to translate exceedance control for ES into a
continuous-time exceedance statement.

A natural question is whether analogous time-uniform exceedance control can be obtained beyond Gaussian perturbations,
where the Brownian embedding in \cref{thm:diag-gauss-embed} is no longer available.
It would also be interesting to see whether one can combine the ideas from \citet{abeille2024when} with our techniques to get tighter regret bounds for action sets with favourable geometry.
Another direction is to understand how far the approach can be pushed for more complex model classes.
Besides these specific questions, we hope that our work will inspire others to explore the use of continuous-time embeddings in the analysis of sequential learning algorithms.

\printbibliography
\end{refsection}
\clearpage

\appendix
\begin{refsection}
\section*{Organisation of the appendix}

The appendix is organized as follows:
First, we give the proof of \cref{lem:regret-via-exceedance} in \cref{app:regret-via-exceedance}.
\cref{appendix:lipschitz-proof} confirms the Lipschitz constant used in the proof of \cref{lem:from-fixed-to-uniform} to move from fixed to uniform control of exceedance fractions.

The next few sections contain material related to the continuous-time embedding and the Brownian exceedance control. We start these sections with a short primer on continuous-time martingales and related notions in \cref{sec:cont-time-prelims}. This is followed by the proof of the embedding result for diagonal Gaussian martingale transforms, \cref{thm:diag-gauss-embed}, in \cref{sec:diag-gauss-embed}.

We then prove the time-uniform Brownian exceedance result in \cref{sec:bm-quantile-K}. This finishes the development of the main continuous-time tools.

The two sections that follow contain proofs needed to finish the proof of \cref{thm:regret-bound}.
In \cref{app:cor-uniform-exceedance}, we prove \cref{cor:uniform-exceedance}, used in the proof of \cref{thm:regret-bound}.
In \cref{app:gamma-bound}, we give a bound on the term $\gamma_{t-1}^{\delta/n}$ used in \cref{thm:regret-bound}.

In \cref{app:diag-gauss-exceedance}, we offer a short discussion of diagonal Gaussian martingale transforms that clarifies why the scalar-clock case is tractable and why the general diagonal case is genuinely harder.

In the final section, \cref{sec:ensemble-size-lb}, we show that with only $d/2$ ensemble members, the regret of ES is linear with constant probability.

\section{Exceedance fraction to regret control (proof of  \cref{lem:regret-via-exceedance})}
\label{app:regret-via-exceedance}
\label{app:optimism-figure-section} %

Here, we prove the following reduction from exceedance frequencies for self-normalised ensemble noise processes to the regret of ensemble sampling, itself used to prove \cref{thm:regret-bound}.

\exceedanceLemma*

The proof is an application of the "master regret bound" of \citet{janz2023ensemble}, which we now state. \cref{fig:optimism-geometry} gives a schematic illustration of the property we are after.

\begin{theorem}[Master regret bound, \citet{janz2023ensemble}]
\label{thm:master-intro}
Fix $\delta \in (0,1)$, $\lambda \geq 1$ and an integer $n \geq 1$. Let $(\cF_t)_{t \geq 1}$ be a filtration. Let $(\theta_t)_{t \geq 1}$ be a $(\cF_t)$-adapted $\Rd$-valued sequence and suppose that
\[
  (\forall t \geq 1) \qquad X_t \in \argmax_{x \in \cX} \langle x, \theta_t \rangle\,.
\]
For each $t \geq 1$, let $V_{t-1} = \lambda I + \sum_{s=1}^{t-1} X_s X_s\tran$ and let $\hat\theta_{t-1}$ be the usual ridge regression estimate of $\theta_\star$, which we assume to be $\cF_{t-1}$-measurable. Let $(b_{t-1})_{t \geq 1}$ be a $(\cF_t)$-predictable sequence of nonnegative random variables such that:
\begin{enumerate}
  \item the event $\cE = \{\forall t \leq n\,,\ \norm{\theta_t - \hat\theta_{t-1}}_{V_{t-1}} \leq b_{t-1}\}$ satisfies $\P{\cE^\complement} \leq \delta$; and
  \item the event $\cE_\star = \{\forall t \leq n\,, \ \norm{\theta_\star - \hat\theta_{t-1}}_{V_{t-1}} \leq b_{t-1}\}$ satisfies $\P{\cE_\star^\complement} \leq \delta$.
\end{enumerate}
Let $(\cA_t)$ be a filtration on the same probability space with
\[
  (\forall t \geq 1) \qquad \cF_{t-1} \subset \cA_{t-1} \subset \cF_t\,,
\]
and for each $t \geq 1$, define the $\cA_{t-1}$-conditional probability of optimism
\[
  p_{t-1} = \P{\langle \theta_t, X_t \rangle \geq \langle x_\star, \theta_\star \rangle \mid \cA_{t-1} }\,.
\]
Then, on an event $\bar{\cE} \subset \cE \cap \cE_\star$ with $\P{\bar{\cE}} \geq 1-3\delta$, for all $\tau \in [n]$, the $\tau$-step regret associated with the action sequence $(X_t)_{t \geq 1}$ satisfies
\begin{equation}
  R_\tau \leq 2 \max_{i \in [\tau]} \frac{b_{i-1}}{p_{i-1}}  \left(2\sqrt{2d \tau \log\left(1 + \frac{\tau}{d\lambda}\right)} + \sqrt{2(4\tau/\lambda + 1)\log\left(\frac{\sqrt{4\tau/\lambda+1}}{\delta}\right)}\right) \,. \label{eq:regret-upper-bound-appendix}
\end{equation}
\end{theorem}

\begin{proof}[Proof of \cref{lem:regret-via-exceedance}]
For each $t \geq 1$, we identify $\theta_t = \theta_{t-1}^{J_t}$ and take $b_{t-1} = \bar\gamma \gamma_{t-1}^{\delta/n} \beta_{t-1}^\delta$. The result will hold on the $1-4\delta$ event given by the intersection of the event of the master theorem and the event where the uniform lower bound on the exceedance fraction holds. We will lower bound each $p_{t-1}$ by $p$ and use that $(b_{t-1})$ is nondecreasing to bound $\max_{i \in [t]} b_{i-1} \leq b_{t-1}$.

We define the filtrations $(\cF_t)$ and $(\cA_t)$ by taking
\[
  \cF_t = \sigma(\zeta^1, \dots, \zeta^m) \vee \sigma((X_s, Y_s, J_s, \xi^1_s, \dots, \xi^m_s) \colon s \leq t)\,,
\]
and $\cA_t = \sigma(\xi_{t+1}^1, \dots, \xi_{t+1}^m) \vee \cF_t$, where $\vee$ denotes the join of $\sigma$-algebras. The nesting property and the measurability requirements on the processes $(\theta_t)$, $(b_{t-1})$, $(\hat\theta_{t-1})$, are met by construction.

First, we confirm that our choice of $(b_{t-1})_{t \geq 1}$ leads to the needed properties for $\cE^\complement$ and $\cE_\star^\complement$. Recall that
\[
  \theta_t = \theta_{t-1}^{J_t}
  = \hat\theta_{t-1} + \bar\gamma \beta_{t-1}^\delta V_{t-1}^{-1} S_{t-1}^{J_t}\,.
\]
Hence
\[
  \norm{\hat\theta_{t-1}-\theta_t}_{V_{t-1}}
  =
  \bar\gamma\beta_{t-1}^\delta
  \norm{S_{t-1}^{J_t}}_{V_{t-1}^{-1}}.
\]
For $t\in[n]$, let
\[
a_n := \sqrt d+\sqrt{2\log(4n/\delta)}
\spaced{and}
r_t := \sqrt{2\log(4n/\delta)+d\log(1+(t-1)/(\lambda d))}.
\]
Define
\[
\cE_t := \left\{
\norm{S_{t-1}^{J_t}}_{V_{t-1}^{-1}}
\le \gamma_{t-1}^{\delta/n}
\right\}
\]
and, for $j\in[m]$,
\[
A_t^j := \{\norm{\zeta^j}>a_n\}
\spaced{and}
B_t^j :=
\left\{
\norm{\sum_{s=1}^{t-1}\xi_s^jX_s}_{V_{t-1}^{-1}}
> r_t
\right\}.
\]
By the definition of $b_{t-1}$, we have
$\bigcap_{t\in[n]}\cE_t \subseteq\cE$.

It remains to bound $\P{(\bigcap_{t\in[n]}\cE_t)^\complement}$. For any $t\in[n]$, using
$V_{t-1}\succeq \lambda I$ and the triangle inequality,
\[
\norm{S_{t-1}^{J_t}}_{V_{t-1}^{-1}}
\le
\norm{\zeta^{J_t}}
+
\norm{\sum_{s=1}^{t-1}\xi_s^{J_t}X_s}_{V_{t-1}^{-1}}.
\]
Since $r_t\ge \sqrt{2\log(4n/\delta)}$, this gives
\[
(\cE_t)^\complement \subseteq A_t^{J_t}\cup B_t^{J_t}.
\]
For every fixed $j\in[m]$, Gaussian concentration gives $\P{A_t^j}\le \delta/(4n)$, while the self-normalised concentration inequality of \citet{abbasi2011improved} gives $\P{B_t^j}\le \delta/(4n)$. Moreover, $A_t^j$ and $B_t^j$ are $\cF_{t-1}$-measurable, and $J_t$ is uniform on $[m]$ conditionally on $\cF_{t-1}$. Hence
\[
\P{A_t^{J_t}}
=
\E\left[\frac1m\sum_{j=1}^m \ind{A_t^j}\right]
=
\frac1m\sum_{j=1}^m \P{A_t^j}
\le \frac{\delta}{4n}
\]
and similarly
\[
\P{B_t^{J_t}}
=
\E\left[\frac1m\sum_{j=1}^m \ind{B_t^j}\right]
=
\frac1m\sum_{j=1}^m \P{B_t^j}
\le \frac{\delta}{4n}.
\]
Therefore $\P{(\cE_t)^\complement}\le \delta/(2n)$ for every $t\in[n]$. Taking a union bound over $t\in[n]$ yields
\[
\P{(\bigcap_{t\in[n]}\cE_t)^\complement}\le \frac{\delta}{2}.
\]
Since $\bigcap_{t\in[n]}\cE_t \subseteq\cE$, this implies $\P{\cE^\complement}\le \delta/2\le\delta$.

Moreover, since
  \[
    \bar\gamma \gamma_{t-1}^{\delta/n} \geq \bar\gamma\sqrt{d} \geq 1\,,
  \]
  we have that $b_{t-1} \geq \beta_{t-1}^\delta$, and therefore $\P{\cE_\star^\complement} \leq \delta$ by \cref{lem:conf_ellip}.

  To conclude, we will establish that on $\cE_\star$, for all $t \geq 1$,
  \[
    p_{t-1} \geq \min_{t\in [n]}\inf_{u\in\Sd} E_{t,m}(u,1/\bar\gamma)\,,
  \]
  with the result then following by intersecting $\bar\cE$ with the event where the right-hand side is lower-bounded by $p$. For this, fix $t \in [n]$. Observe that by the choice of $X_t$, $\langle \theta_t, X_t \rangle \geq \langle \theta_t, x_\star \rangle$. On $\cE_\star$,
  \[
    \langle \theta_\star - \hat\theta_{t-1}, x_\star \rangle \leq \norm{\theta_\star - \hat \theta_{t-1}}_{V_{t-1}} \norm{x_\star}_{V_{t-1}^{-1}} \leq \beta_{t-1}^\delta \norm{x_\star}_{V_{t-1}^{-1}}\,.
  \]
  Therefore, on $\cE_\star$,
  \[
    \langle \theta_t - \hat\theta_{t-1}, x_\star \rangle \geq \beta_{t-1}^\delta \norm{x_\star}_{V_{t-1}^{-1}}
    \implies \langle \theta_t, X_t \rangle \geq \langle \theta_\star, x_\star \rangle\,.
  \]
If $x_\star = 0$, then $0 \in \cX$, and therefore $\langle \theta_t, X_t \rangle = \max_{x \in \cX} \langle x, \theta_t \rangle \geq 0 = \langle \theta_\star, x_\star \rangle$, so $p_{t-1} = 1$ and the result holds (indeed, without the need for the $1/p$ factor). Assume henceforth that $x_\star \neq 0$ and define $u_\star = V_{t-1}^{-1} x_\star$. Then $\norm{u_\star}_{V_{t-1}} = \norm{x_\star}_{V_{t-1}^{-1}}$, and using $\theta_t = \theta_{t-1}^{J_t}$ we obtain
  \[
    \langle \theta_t - \hat\theta_{t-1}, x_\star \rangle
    = \bar\gamma\beta_{t-1}^\delta\langle V_{t-1}^{-1} S^{J_t}_{t-1}, x_\star \rangle
    = \bar\gamma\beta_{t-1}^\delta \langle S^{J_t}_{t-1}, u_\star \rangle\,.
  \]
  Hence, if
  \[
    \frac{\langle S_{t-1}^{J_t}, u_\star \rangle}{\norm{u_\star}_{V_{t-1}}} \geq \frac{1}{\bar\gamma}\,,
  \]
  then
  \[
    \langle \theta_t - \hat\theta_{t-1}, x_\star \rangle
    \geq \beta_{t-1}^\delta \norm{x_\star}_{V_{t-1}^{-1}},
  \]
  which, by the previous implication on $\cE_\star$, yields
  \[
    \langle \theta_t, X_t \rangle \geq \langle \theta_\star, x_\star \rangle\,.
  \]
  Now note that $\theta_{t-1}^1, \dots, \theta^m_{t-1}$ and $u_\star$ are $\cA_{t-1}$-measurable and that conditionally on $\cA_{t-1}$, $J_t$ is uniform on $[m]$. Thus,
  \begin{align*}
        p_{t-1}
        &= \P{\langle \theta_t, X_t \rangle \geq \langle x_\star, \theta_\star \rangle \mid \cA_{t-1} } \\
        &\geq  \P{\langle S^{J_t}_{t-1}, u_\star \rangle /\norm{u_\star}_{V_{t-1}} \geq 1/\bar{\gamma} \mid \cA_{t-1} } \\
        &= E_{t,m}(u_\star, 1/\bar\gamma) \\
        &\geq \min_{t \in [n]} \inf_{u \in \Rd\setminus \{0\}} E_{t,m}(u, 1/\bar\gamma) \\
        &= \min_{t \in [n]} \inf_{u \in \Sd} E_{t,m}(u, 1/\bar\gamma)\,,
  \end{align*}
where the final equality uses that $E_{t,m}(u,c)$ is invariant under rescaling of $u$.
\end{proof}

\begin{figure}[p]
\centering
\includegraphics[width=0.75\linewidth]{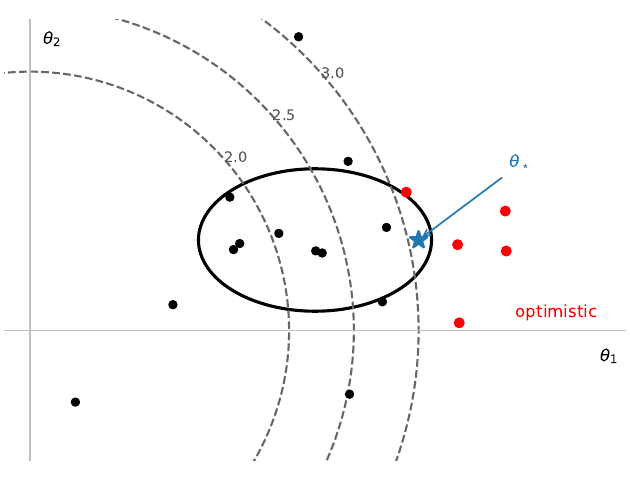}
\caption{Schematic geometry for $\cX = \Bd$:
the confidence ellipsoid (solid) is off-center and intersects level sets of the optimal reward $J(\theta) = \max_{x\in \cX} \ip{x,\theta} = \norm{\theta}$ (dashed).
For illustration purposes,
a sample of $m$ perturbed model parameters is drawn from a Gaussian
with covariance inflated by about $1.2\times$ relative to the ellipsoid axes;
those whose optimal reward exceeds that of the true parameter are highlighted as optimistic (red coloured dots).
The master regret bound states that if we can keep the exceedance frequency of optimistic samples above a constant, then the regret will be under control.
The difficulty with ensemble sampling is that, unlike in Thompson sampling, the model parameters are not freshly drawn in each time step but evolve in a correlated fashion over time. As such, controlling the exceedance frequencies is more challenging, and this is the main technical contribution of the paper.
}
\label{fig:optimism-geometry}
\end{figure}

\section{Lipschitzness of self-normalised martingale on unit sphere (proof of \cref{lem:from-fixed-to-uniform})}
\label{appendix:lipschitz-proof}
\fromFixedToUniform*

\begin{proof}[Proof of \cref{lem:from-fixed-to-uniform}]
    To argue for Lipschitzness for a fixed $u\in \Sd$,
  we separate out the self-normalized $V_{t-1}^{-1/2} S_{t-1}^j$ term from the rest:
  \[
  \frac{\ip{u,S_{t-1}^j}}{\norm{u}_{V_{t-1}}}
  = \ip*{ \frac{V_{t-1}^{1/2}u}{\norm{V_{t-1}^{1/2}u}}, V_{t-1}^{-1/2} S_{t-1}^j}\,.
  \]
  By a Cauchy--Schwarz argument, the Lipschitz constant of the map under investigation is at most
  the product of $\norm{V_{t-1}^{-1/2} S_{t-1}^j}$
  and the Lipschitz constant of the map $u \mapsto V_{t-1}^{1/2}u/\norm{V_{t-1}^{1/2}u}$.
  By the choice of $\gamma_{t-1}^{\delta/m}$, a union bound, with probability at least $1-\delta/2$,
  for all $t\in [n]$ and $j\in [m]$, $\norm{V_{t-1}^{-1/2} S_{t-1}^j} \le \gamma_{n-1}^{\delta/m}$.
  As to the Lipschitzness of the second map,
  the inequality
  \[
  \norm*{\frac{a}{\norm{a}} - \frac{b}{\norm{b}}} \le \frac{2\norm{a-b}}{\min\{\norm{a},\norm{b}\}}
  \]
  that holds for any $a,b\in \Rd$ nonzero vectors gives that this map
  is $2\norm{V_{t-1}^{1/2}}/\sqrt{\lambda}$-Lipschitz on $\Sd$.
  Now, since $V_{t-1}$ is positive definite,
  $\norm{V_{t-1}^{1/2}} = \sqrt{\norm{V_{t-1}}}$ and a standard calculation gives that
  $\norm{V_{t-1}} \le \lambda + \sum_{s=1}^{t-1} \norm{X_s}^2 \le \lambda + n$.
  Putting things together and a union bound gives the result.
\end{proof}

\clearpage

\section{Continuous-time preliminaries}
\label{sec:cont-time-prelims}
We will use a few standard notions from continuous-time martingale theory, and this short note fixes notation. First, $a\wedge b$ denotes the minimum of $a$ and $b$.
Many continuous-time texts index time with subscripts; we follow that convention in this preliminaries section and in quoted results.
In our own proofs we often write the continuous-time variable as an argument (e.g., $W(t)$) to avoid clashes with discrete-time subscripts; we mix these notations and hope the context makes the time variable clear (see, e.g., the OU lemma).

For this subsection only, we write $(\mathcal F_t)_{t\ge 0}$ for the continuous-time filtration; it is unrelated to the discrete-time $\cF_t$ used elsewhere.
For the definitions below we only need that $(\mathcal F_t)$ is nondecreasing; if desired, one can assume the usual conditions (right-continuity and completeness).
An adapted process $M=(M_t)_{t\ge 0}$ is one with $M_t$ being $\mathcal F_t$-measurable for each $t$.
A (continuous-time) martingale is an adapted process with $\E|M_t|<\infty$ for all $t$ and $\E[M_t\mid\mathcal F_s]=M_s$ for $s\le t$.
A continuous martingale is a martingale with continuous paths.
A local martingale is an adapted process for which there exists an increasing sequence of stopping times $\tau_n\uparrow\infty$ such that, for each $n$, the stopped process $(M_{t\wedge \tau_n})_{t\ge 0}$ is a martingale with respect to the same filtration; this is weaker than being a martingale and captures processes that behave like martingales up to larger and larger random horizons. Local martingales are the natural class in continuous-time analysis.
A continuous local martingale is a local martingale with continuous paths.
A standard Brownian motion $W=(W_t)_{t\ge 0}$ has $W_0=0$, continuous paths, independent, Gaussian increments with $W_t-W_s\sim \mathcal N(0,t-s)$ for $0\le s\le t$; in particular, it is a martingale (hence a local martingale).
For Gaussian processes, the law is determined by the covariance function; in particular, a centered Gaussian process with covariance $\E[B_sB_t]=s\wedge t$ is a standard Brownian motion.
For a continuous local martingale $M$, its quadratic variation $\langle M\rangle_t$ can be defined as the limit in probability of sums of squared increments along partitions of $[0,t]$; it is an increasing process and the continuous-time analogue of accumulated conditional variance (indeed, $M_t^2-\langle M\rangle_t$ is a local martingale).
For Brownian motion, $\langle W\rangle_t=t$.
For two continuous local martingales $M,N$, the covariation $\langle M,N\rangle_t$ is defined similarly via products of increments, and we write $\langle M,N\rangle$ for the process $(\langle M,N\rangle_t)_{t\ge 0}$ and $MN$ for the pointwise product process $(M_tN_t)_{t\ge 0}$; strong orthogonality means $\langle M,N\rangle$ is identically zero as a process.
For background and proofs, see \citet[Ch.~I]{RevuzYor1999} for Brownian motion and filtrations, \citet[Ch.~II]{RevuzYor1999} for martingales, \citet[Ch.~IV]{RevuzYor1999} for quadratic variation and covariation, and \citet[Ch.~V]{RevuzYor1999} for time changes and the DDS/Knight results.

\section{A Representation Result for Diagonal Gaussian Martingale Transforms}
\label{sec:diag-gauss-embed}
\theoremDiagGaussEmbedding*

We start with some auxiliary results.

\begin{lemma}[Pinned Brownian segment]\label{lem:pinned-bm-seg}
Fix $\Delta>0$ and let $Z\sim \mathcal N(0,\Delta)$. On an extension of the probability space,
there exists a standard Brownian motion $(B_s)_{0\le s\le \Delta}$ such that $B_\Delta = Z$.
\end{lemma}

\begin{proof}
Let $\widetilde B$ be a standard Brownian motion on $[0,\Delta]$, independent of $Z$, and define
\[
B(t)\;:=\; \widetilde B(t) - \frac{t}{\Delta}\widetilde B(\Delta) + \frac{t}{\Delta}Z\,, \qquad t\in[0,\Delta]\,.
\]
Then $B(0)=0$ and $B(\Delta)=Z$. The process $B$ is continuous and Gaussian because it is an affine transform of the
jointly Gaussian family $(Z, (\widetilde B(t))_{t\in[0,\Delta]})$. Moreover, for $0\le s\le t\le\Delta$,
\begin{align*}
  \Cov(B(s),B(t)) &= \frac{st}{\Delta^2}\Var Z
  + \Cov\!\Big(\widetilde B(s) - \frac{s}{\Delta} \widetilde B(\Delta),\, \widetilde B(t) - \frac{t}{\Delta} \widetilde B(\Delta)\Big) \\
  &= \Cov(\widetilde B(s), \widetilde B(t)) = s \wedge t \,.
\end{align*}
Since $B$ has a centered Gaussian law with covariance $s\wedge t$, it is a standard Brownian motion.
\end{proof}

\begin{remark}
\cref{lem:pinned-bm-seg} may be viewed as sampling $Z$ and then sampling (conditionally on $Z$)
a Brownian bridge from $0$ to $Z$ over $[0,\Delta]$ (i.e., a pinned Brownian segment).
\end{remark}

\begin{theorem}[Dambis--Dubins--Schwarz {\cite[Ch.~V, Thm.~1.6--1.7]{RevuzYor1999}}]\label{thm:DDS}
  Let $(M_t,\mathcal F_t)_{t\ge 0}$ be a continuous local martingale with $M_0=0$ and quadratic variation
  $\langle M\rangle_t$. Define the right-continuous inverse
  \[
  \tau(s) \ :=\ \inf\{t\ge 0:\ \langle M\rangle_t \ge s\},\qquad s\ge 0\,,
  \]
  and set $\tau(\infty):=\sup_{s\ge 0}\tau(s)$.
  Then, there exists an enlargement of the filtered
  probability space $(\Omega,\mathcal F,(\mathcal F_{\tau(s)})_{s\ge 0})$ and a standard Brownian motion
  $(W_t)_{t\ge 0}$ on the enlarged space such that for any $s\ge 0$,
  \[
  W_s \ =\ M_{\tau(s)} \text{ on the event } \{0\le s <\langle M\rangle_\infty\}
  \cup\{ s = \langle M\rangle_\infty < \infty\}\,
  \,,
  \]
  Consequently, for any $t\ge 0$,
  \begin{align}
    M_t \ =\ W_{\langle M\rangle_t}
    \text{ on the event } \{ 0\le t<\tau(\langle M\rangle_\infty)\}\cup \{
      t=\tau(\langle M\rangle_\infty),\, \langle M\rangle_\infty<\infty
     \},
    \label{eq:revdds}
  \end{align}
  where for the last display
  we let $M_\infty:=\lim_{t\to\infty}M_t$ on the set $\{ \langle M\rangle_\infty<\infty\}$ and define it arbitrarily otherwise.
\end{theorem}
Above, $M_\infty$ is well-defined by Proposition~IV.1.26 of \citet{RevuzYor1999}. The identity \eqref{eq:revdds}
follows from elementary arguments.
The ``enlargement'' is only needed to handle the case where $\langle M\rangle_\infty<\infty$ with positive probability.
Intuitively, the local martingale determines the Brownian motion up to the time $\langle M\rangle_\infty$, after which we just extend the Brownian motion independently.
In our case, $\langle M\rangle_\infty<c_\infty$ for some finite constant $c_\infty>0$ with probability one
and we will rely on accessing the Brownian motion beyond the time $\langle M\rangle_\infty$ up to time $c_\infty$.

Note that the above theorem does not define the Brownian motion $W$ uniquely, since it only specifies its values up to time $\langle M\rangle_\infty$.
As such, in what follows we will call any process $W$ that satisfies the statement of the theorem a \emph{Dambis--Dubins--Schwarz (DDS) Brownian motion} associated with $M$.

\begin{theorem}[Knight's theorem (orthogonal martingales) {\cite[Ch.~V, Thms.~1.9--1.10]{RevuzYor1999}}]\label{thm:knight}
  Let $M^1,\dots,M^m$ each be continuous local martingales as in \cref{thm:DDS}.
  Assume that the $M^j$ are pairwise strongly orthogonal: $\langle M^j,M^k\rangle\equiv 0$ for all $j\neq k$.
  Then, the associated DDS Brownian motions $W^1,\dots,W^m$ can be chosen to be a standard $m$-dimensional Brownian motion (i.e., with independent components).
\end{theorem}

\begin{proof}[Proof of \cref{thm:diag-gauss-embed}]
Fix $n\in\N$. Work on an extension on which we can construct, for each $t\in\{0,1,\dots,n-1\}$ and each $j\in[m]$,
a standard Brownian motion segment
\[
B_{t}^{(j)}:\ [0,1]\to\R
\qquad\text{such that}\qquad
B_{t}^{(j)}(1)=\xi_t^{(j)},
\]
using \cref{lem:pinned-bm-seg} with $\Delta=1$.
Assume these segments are mutually independent across pairs $(t,j)$ \emph{conditional on the endpoints}
$\{B_t^{(j)}(1)\}_{t,j}=\{\xi_t^{(j)}\}_{t,j}$, and that the additional randomness used to sample the paths
given their endpoints (equivalently, the associated Brownian bridge innovations) is independent of the
original $\sigma$-field.

Fix $j\in[m]$ and define a continuous-time process $M^{(j)}(r)$ on $[0,n]$ by stitching scaled segments:
set $M^{(j)}(0):=0$ and for $r\in[t,t+1]$ define
\[
M^{(j)}(r)
\;:=\;
M^{(j)}(t) + D_{t,j}\, B_{t}^{(j)}\!\big(r-t\big).
\]
Then, for each integer $t\in\{0,1,\dots,n\}$,
\[
M^{(j)}(t)
=
\sum_{s=0}^{t-1} D_{s,j}\, \xi_s^{(j)},
\]
with the empty sum equal to $0$. In particular, for each $t\in\{0,1,\dots,n-1\}$,
\[
M^{(j)}(t+1)=\sum_{s=0}^{t} D_{s,j}\, \xi_s^{(j)} = M_t^{(j)}.
\]

Let $(\mathcal G_r)_{r\in[0,n]}$ be the filtration generated by $(\mathcal F_{\lfloor r\rfloor-1})$ (with $\mathcal F_{-1}$ trivial) together with
the Brownian segments revealed progressively in $r$. That is,
for $r\in[0,n]$,
\[
\mathcal G_r
\;:=\;
\sigma\!\Big(
\mathcal F_{\lfloor r\rfloor-1}
\;\cup\;
\big\{B_{t}^{(k)}(u):\, k\in[m],\ t\in\{0,\dots,\lceil r\rceil-1\},\ 0\le u\le (r-t)_+\wedge 1\big\}
\Big),
\]
where $(x)_+:=\max\{x,0\}$. Since $D_{t,j}$ is $\mathcal F_{t-1}$-measurable and
$B_t^{(j)}(\cdot)-B_t^{(j)}(0)$ has independent increments and is independent of $\mathcal F_{t-1}$,
the process $M^{(j)}$ is a continuous local martingale w.r.t.\ $(\mathcal G_r)$ with quadratic variation
\[
\langle M^{(j)}\rangle_r
=
\sum_{s=0}^{t-1} D_{s,j}^2 \;+\; D_{t,j}^2\,(r-t),
\qquad r\in[t,t+1].
\]
In particular, at integer times $t\in\{0,1,\dots,n-1\}$,
\(
\langle M^{(j)}\rangle_{t+1} = A_{t,j}^2.
\)

Extend $(M^{(j)}(r),\mathcal G_r)$ to $r\ge n$ by setting $M^{(j)}(r):=M^{(j)}(n)$ and
$\mathcal G_r:=\mathcal G_n$ for $r\ge n$. Then $M^{(j)}$ is a continuous local martingale on $[0,\infty)$
with $\langle M^{(j)}\rangle_\infty=\langle M^{(j)}\rangle_n=A_{n-1,j}^2$, so \cref{thm:DDS} applies.

We will now apply Knight's theorem (\cref{thm:knight}) to obtain the $m$ independent Brownian motions $W^1,\dots,W^m$ from the $m$ continuous local martingales $M^{(1)},\dots,M^{(m)}$ that start at zero.
The strong orthogonality condition is immediate from the construction:
for $j\neq k$ the driving Brownian segments $\{B_t^{(j)}\}_{t\in\{0,\dots,n-1\}}$ and $\{B_t^{(k)}\}_{t\in\{0,\dots,n-1\}}$
are independent, so $\langle M^{(j)},M^{(k)}\rangle\equiv 0$.
Thus, the theorem applies and we obtain independent standard Brownian motions $W^1,\dots,W^m$ such that for all $r\in [0,n]$,
\[
M^{(j)}(r) = W^j\!\big(\langle M^{(j)}\rangle_r\big),
\]
hence in particular $M_t^{(j)} = M^{(j)}(t+1) = W^j(A_{t,j}^2)$ for all $t\in\{0,1,\dots,n-1\}$.
This verifies the hypotheses of Knight's theorem for orthogonal continuous local martingales (\cref{thm:knight}), so the DDS Brownian motions $(W^1,\dots,W^m)$ are independent standard Brownian motions. This completes the proof.
\end{proof}

\section{Time-uniform exceedance count for Brownian motions}
\label{sec:bm-quantile-K}
\theoremExcBrownian*

The proof follows a more or less standard path that is usually used to get the law of iterated logarithm for Brownian motion,
\citet[e.g.,][Ch. II, Theorem 9.23]{karatzasshreve1991brownian}. See also \citet{darlingerdos1956} for an early reference where the embedding technique is also used.

We will use a specific scaling/time-change of Brownian motion---often called the Lamperti transform---to obtain an Ornstein--Uhlenbeck (OU) process.
In particular, for a standard Brownian motion $W=(W(t))_{t\ge 0}$, we define the (standard) OU process by
$U(s)=e^{-s/2}W(e^s)$ for $s\in\R$. This is a centered Gaussian process with covariance $\E[U(s)U(t)]=e^{-|s-t|/2}$ (so $\Var(U(s))=1$).
We also extend the class of OU processes by allowing them to be started from an arbitrary point $x\in\R$, by which we mean the process that
takes the form
$U_x(u)=e^{-u/2}x+e^{-u/2}W(e^u-1)$ for $u\ge 0$ (so $U_x(0)=x$),
where $W$ is a standard Brownian motion.
It is not hard to see that both an OU and an OU started from $x$ are Gaussian processes and thus are uniquely determined by their means and covariances.
We will need the following result:

\begin{lemma}[OU transition via Lamperti transform]\label{lem:ou-transition}
Let $(W(t))_{t\ge 0}$ be a standard Brownian motion and consider the OU process
\[
U(s) := e^{-s/2}W(e^s),\qquad s\in\R.
\]
Fix $s_0\in\R$ and set $\mathcal F_{s_0}:=\sigma\{W(t):0\le t\le e^{s_0}\}$. Define for $t\ge 0$,
\begin{align}
\widetilde W_{s_0}(t)
:= e^{-s_0/2}\bigl(W(e^{s_0}(1+t)) - W(e^{s_0})\bigr),
\end{align}
and for $u\ge 0$,
\begin{align}
\widetilde U_{s_0}(u) := e^{-u/2}\,\widetilde W_{s_0}(e^u-1). \label{eq:ou-tilde}
\end{align}
Then $\widetilde W_{s_0}$ is a standard Brownian motion independent of $\mathcal F_{s_0}$ and, for every $u\ge 0$,
\[
U(s_0+u)=e^{-u/2}U(s_0)+\widetilde U_{s_0}(u).
\]
Moreover, $(\widetilde U_{s_0}(u))_{u\ge 0}$ is an Ornstein--Uhlenbeck process started at $0$ and independent of $U(s_0)$.
\end{lemma}

\begin{proof}
The process $\widetilde W_{s_0}$ is continuous with $\widetilde W_{s_0}(0)=0$, and for $0\le t_1<t_2$,
\[
\widetilde W_{s_0}(t_2)-\widetilde W_{s_0}(t_1)
=
e^{-s_0/2}\bigl(W(e^{s_0}(1+t_2))-W(e^{s_0}(1+t_1))\bigr).
\]
By independent increments of Brownian motion, the increments above are independent and Gaussian with mean $0$ and
variance $t_2-t_1$, so $\widetilde W_{s_0}$ is a standard Brownian motion. It is measurable with respect to increments
of $W$ after time $e^{s_0}$, hence independent of $\mathcal F_{s_0}$.

Using $e^{s_0+u}=e^{s_0}(1+(e^u-1))$, we compute
\[
U(s_0+u)
=
e^{-(s_0+u)/2}W(e^{s_0+u})
=
e^{-u/2}U(s_0)+e^{-u/2}\widetilde W_{s_0}(e^u-1)
=
e^{-u/2}U(s_0)+\widetilde U_{s_0}(u).
\]
Since we have verified that $\widetilde W_{s_0}$ is a standard Brownian motion independent of $\mathcal F_{s_0}$,
by its definition, $\widetilde U_{s_0}$ is an OU process started at $x=0$ and is independent of $U(s_0)$.
Indeed, $U(s_0)$ is $\mathcal F_{s_0}$-measurable while $\widetilde U_{s_0}$ is measurable with respect to $\widetilde W_{s_0}$,
which is independent of $\mathcal F_{s_0}$.
\end{proof}
We record another standard result.
\begin{lemma}[Brownian maximum tail bound]\label{lem:bm-maximum-tail}
Let $(W(t))_{t\ge 0}$ be a standard Brownian motion. Then, for any $T>0$ and any $a>0$,
\[
\mathbb{P}\Big(\sup_{t\in[0,T]} |W(t)| \ge a\Big)
\le
4\exp\Big(-\frac{a^2}{2T}\Big).
\]
\end{lemma}
\begin{proof}
By the reflection principle for Brownian motion (see, e.g., \citet[Ch.~III, Prop.~3.7]{RevuzYor1999}),
\[
\mathbb{P}\Big(\sup_{t\in[0,T]} W(t) \ge a\Big)
= 2\mathbb{P}(W(T)\ge a) \;=\; 2\big(1-\Phi(a/\sqrt{T})\big).
\]
By symmetry of Brownian motion,
\[
\mathbb{P}\Big(\sup_{t\in[0,T]} |W(t)| \ge a\Big)
\leq 2\,\mathbb{P}\Big(\sup_{t\in[0,T]} W(t) \ge a\Big)
= 4\big(1-\Phi(a/\sqrt{T})\big).
\]
The standard Gaussian tail bound $1-\Phi(x)\le e^{-x^2/2}$ (e.g., Chernoff) for $x>0$ gives the result.
\end{proof}

\begin{proof}[Proof of \cref{thm:bm-quantile-K}]
We present the proof in two parts: first we prove the main claim up to the furthermore clause,
then we prove the furthermore clause.
\paragraph{Proof of the main claim.}
Fix $0<\tau < \tau'<\infty$, $c>0$, and $0<p<\frac14(1-\Phi(c))$.
Define
\[
\varepsilon \;:=\; \frac{\Phi^{-1}(1-4p)-c}{3} \;>\;0,
\]
so that $1-\Phi(c+3\varepsilon)=4p$.
Since $\varepsilon>0$, we can choose a discretization $h\in(0,1]$ small enough so that
\begin{equation}\label{eq:h-constraints}
e^{-h/2}(c+3\varepsilon)\ge c+2\varepsilon
\qquad\text{and}\qquad
4\exp\!\Big\{-\frac{2\varepsilon^2}{e^h-1}\Big\}\le \frac12.
\end{equation}
Let $\sigma:=\log\tau$, $\sigma':=\log\tau'$, and
\[
K \;:=\; \Big\lceil \frac{\sigma'-\sigma}{h}\Big\rceil
 \;=\; \Big\lceil \frac{\log(\tau'/\tau)}{h}\Big\rceil,
\qquad
s_k := \sigma + kh\quad (k=0,1,\dots,K),
\]
so that $[\sigma,\sigma']\subseteq \bigcup_{k=0}^{K-1}[s_k,s_{k+1}]$.

For each $j\in[m]$, define the time-changed process
\[
U_j(s)\;:=\; e^{-s/2}W^j(e^s), \qquad s\in\mathbb{R}.
\]
Then $(U_j)_{j=1}^m$ are independent standard Ornstein--Uhlenbeck processes (see \cref{lem:ou-transition}), and for $t\in[\tau,\tau']$
(with $s=\log t\in[\sigma,\sigma']$),
\[
W^j(t)\ge c\sqrt{t}
\;\Longleftrightarrow\;
U_j(\log t)\ge c.
\]
Thus it suffices to prove that with probability at least $1-\delta$,
\[
\forall s\in[\sigma,\sigma']:\quad \#\{j\in[m]: U_j(s)\ge c\}\ge pm.
\]

For each $k\in\{0,\dots,K-1\}$ define
\[
S_k^j := \mathbf{1}\Big\{\forall s\in[s_k,s_{k+1}]: U_j(s)\ge c\Big\},
\qquad
S_k := \sum_{j=1}^m S_k^j.
\]
If we show that $S_k\ge pm$ for all $k$, then the desired bound holds for all $s\in[\sigma,\sigma']$,
hence for all $t\in[\tau,\tau']$.

Fix $k\in\{0,\dots,K-1\}$. By \cref{lem:ou-transition}, for any $u\in[0,h]$,
\[
U_j(s_k+u) \;=\; e^{-u/2}U_j(s_k) + \widetilde U_{j,k}(u),
\]
where $\widetilde U_{j,k}$ is an OU process started at $0$ at time $0$, independent of $U_j(s_k)$.
Define the event indicators
\[
A_k^j := \mathbf{1}\{U_j(s_k)\ge c+3\varepsilon\},
\qquad
E_k^j := \mathbf{1}\Big\{\forall u\in[0,h]: \widetilde U_{j,k}(u)\ge -2\varepsilon\Big\}.
\]
By \eqref{eq:h-constraints}, for all $u\in[0,h]$,
\[
e^{-u/2}(c+3\varepsilon)-2\varepsilon \;\ge\; e^{-h/2}(c+3\varepsilon)-2\varepsilon \;\ge\; c,
\]
so $A_k^j E_k^j=1$ implies $S_k^j=1$. Hence, writing $Z_k^j:=A_k^jE_k^j$ and $Z_k:=\sum_{j=1}^m Z_k^j$,
we have $S_k\ge Z_k$.

Now $\mathbb{P}(A_k^j=1)=\mathbb{P}(N(0,1)\ge c+3\varepsilon)=1-\Phi(c+3\varepsilon)=4p$.
Let $q:=\mathbb{P}(E_k^j=0)$ (this does not depend on $j,k$). Since $A_k^j$ and $E_k^j$ are independent
(and independent across $j$), $Z_k\sim \mathrm{Binomial}(m,r)$ with
\[
r = (4p)(1-q).
\]

We upper bound $q$.
By \cref{lem:ou-transition}, we may represent an OU process started at $0$ as
$\widetilde U_{j,k}(u)=e^{-u/2}\widetilde W_{j,k}(e^u-1)$ for a standard Brownian motion $\widetilde W_{j,k}$.
Therefore,
\begin{align*}
q
&= \mathbb{P}\Big(\inf_{0\le u\le h}\widetilde U_{j,k}(u)<-2\varepsilon\Big)
 \;\le\; \mathbb{P}\Big(\sup_{0\le u\le h}|\widetilde U_{j,k}(u)|\ge 2\varepsilon\Big) \\
&\le \mathbb{P}\Big(\sup_{0\le t\le e^h-1}|\widetilde W_{j,k}(t)|\ge 2\varepsilon\Big)
 \;\le\; 4\exp\!\Big\{-\frac{(2\varepsilon)^2}{2(e^h-1)}\Big\}
 \;=\; 4\exp\!\Big\{-\frac{2\varepsilon^2}{e^h-1}\Big\}\,,
\end{align*}
where we used \cref{lem:bm-maximum-tail} for the last inequality.
By \eqref{eq:h-constraints}, $q\le \tfrac12$, hence $r\ge (4p)\cdot \tfrac12 = 2p$.

By the multiplicative Chernoff bound for binomials \citep[e.g.,][Theorem~1.1]{dubhashi_panconesi_2009},
\[
\mathbb{P}\Big(Z_k \le \frac{rm}{2}\Big) \;\le\; \exp\!\Big\{-\frac{rm}{8}\Big\}
\;\le\; \exp\!\Big\{-\frac{pm}{4}\Big\},
\]
where we used $r\ge 2p$. On the event $\{Z_k \ge rm/2\}$ we have $Z_k\ge pm$, hence also $S_k\ge pm$.

Finally, by a union bound over $k=0,\dots,K-1$,
\[
\mathbb{P}\Big(\exists k\in\{0,\dots,K-1\}: S_k<pm\Big)
\;\le\;
K \exp\!\Big\{-\frac{pm}{4}\Big\}.
\]
Under the assumption $m\ge \frac{4}{p}\log\!\big(\frac{K}{\delta}\big)$, the right-hand side is at most $\delta$.
Thus, with probability at least $1-\delta$, for every $k$ we have $S_k\ge pm$, and consequently for all
$s\in[\sigma,\sigma']$,
$\#\{j\in[m]:U_j(s)\ge c\}\ge pm$.
Translating back to $t=e^s$ gives the claimed statement for $(W^j(t))_{j\le m}$ on $[\tau,\tau']$.

\paragraph{Proof of the furthermore clause}
Inspecting the argument above, we see that the only place where the choice of $h$ matters is through
the two constraints in \cref{eq:h-constraints}. We now make these constraints explicit.

Recall $\varepsilon := (\Phi^{-1}(1-4p)-c)/3>0$, so that $1-\Phi(c+3\varepsilon)=4p$.

\smallskip
\noindent\textbf{(i) The drift/contractivity constraint}\ \ \
The condition $e^{-h/2}(c+3\varepsilon)\ge c+2\varepsilon$ is equivalent to
\[
e^{-h/2} \;\ge\; \frac{c+2\varepsilon}{c+3\varepsilon}
\qquad\Longleftrightarrow\qquad
h \;\le\; 2\log\!\Big(\frac{c+3\varepsilon}{c+2\varepsilon}\Big).
\]

\smallskip
\noindent\textbf{(ii) The short-interval fluctuation constraint}\ \ \
The condition $4\exp\{-2\varepsilon^2/(e^h-1)\}\le \tfrac12$ is equivalent to
\[
\exp\!\Big\{-\frac{2\varepsilon^2}{e^h-1}\Big\}\le \frac18
\qquad\Longleftrightarrow\qquad
\frac{2\varepsilon^2}{e^h-1}\ge \log 8
\qquad\Longleftrightarrow\qquad
h \;\le\; \log\!\Big(1+\frac{2\varepsilon^2}{\log 8}\Big).
\]

Therefore, defining
\[
h_{\star} \;:=\; \min\!\Big\{1,\; 2\log\!\Big(\frac{c+3\varepsilon}{c+2\varepsilon}\Big),\;
\log\!\Big(1+\frac{2\varepsilon^2}{\log 8}\Big)\Big\},
\]
we have that any $h\in(0,h_\star]$ satisfies \cref{eq:h-constraints}, and hence the conclusion of
\cref{thm:bm-quantile-K} holds for any such $h$.

\smallskip
\noindent\textbf{Claim: $c\le 1/20$, $p\le 1/10$ implies $h=1/250$ is admissible.}
Assume $c\le 1/20$ and $p\le 1/10$. Then $1-4p\ge 0.6$, so
$\Phi^{-1}(1-4p)\ge \Phi^{-1}(0.6)\ge 1/4$, and hence
\[
\varepsilon \;=\; \frac{\Phi^{-1}(1-4p)-c}{3}
\;\ge\; \frac{1/4-1/20}{3}
\;=\; \frac{1}{15}.
\]
For the drift/contractivity bound, note that $(c+3\varepsilon)/(c+2\varepsilon)=1+\varepsilon/(c+2\varepsilon)$
is decreasing in $c$ and increasing in $\varepsilon$ for $c,\varepsilon>0$, so using $c\le 1/20$ and
$\varepsilon\ge 1/15$ we compute
\[
2\log\!\Big(\frac{c+3\varepsilon}{c+2\varepsilon}\Big)
\;\ge\;
2\log\!\Big(\frac{1/20 + 3\cdot(1/15)}{1/20 + 2\cdot(1/15)}\Big)
=
2\log\!\Big(\frac{1/4}{11/60}\Big)
=
2\log\!\Big(\frac{15}{11}\Big)
>\; 0.6
>\; \frac{1}{250}.
\]
For the fluctuation bound, using $\varepsilon\ge 1/15$,
\[
\log\!\Big(1+\frac{2\varepsilon^2}{\log 8}\Big)
\;\ge\;
\log\!\Big(1+\frac{2}{225\log 8}\Big).
\]
Since $\log 8 < 2.1$, we have $\frac{2}{225\log 8} > \frac{2}{225\cdot 2.1} > 0.0042$ and hence,
using $\log(1+x)\ge x/(1+x)$ for $x>0$,
\[
\log\!\Big(1+\frac{2\varepsilon^2}{\log 8}\Big)
\;\ge\;
\frac{0.0042}{1+0.0042}
>\; 0.004
=\; \frac{1}{250}.
\]
Thus $h_\star \ge 1/250$, so the choice $h=1/250$ is admissible.
\end{proof}

\section{Proof of \cref{cor:uniform-exceedance}}
\label{app:cor-uniform-exceedance}
\corUniformExceedance*
Before the proof, we record a simple log-inequality and then complete the proof.
\begin{lemma}[Solving a log-inequality]\label{lem:solve-log-ineq}
  Let $A>0$ and $B\ge 0$. If $m$ is positive and
  \[
  m \;\ge\; 2A\log A \;+\; 2B,
  \]
  then $m \ge A\log m + B$.
  \end{lemma}

  \begin{proof}
  This is an immediate corollary of Proposition~4 in \citep{AllocationTCS10}:
  apply it to the inequality $m/A - B/A \ge \log m$.
  \end{proof}

\begin{proof}[Proof of \cref{cor:uniform-exceedance}]
Set $p:=1/10$ and $c:=2/\bar\gamma=1/20$.
Since $p_0(1/20)=\tfrac14(1-\Phi(1/20))>1/10$, the condition
$p<p_0(2/\bar\gamma)$ in \cref{lem:from-fixed-to-uniform} is satisfied.
Define
\[
L := 2\,\gamma_{n-1}^{\delta/m}\sqrt{1+n/\lambda},
\qquad
\epsilon := \frac{1}{\bar\gamma L},
\qquad
\delta' := \frac{\delta/2}{(3/\epsilon)^d}.
\]
By \cref{lem:from-fixed-to-uniform}, it suffices to show $m \ge m_0(c,\delta',p)$.
With $K:=\lceil 250\log((\lambda+n)/\lambda)\rceil$ and $p=1/10$,
\[
m_0(c,\delta',p)
=40\Bigg[\log\Big(\frac{2K}{\delta}\Big)+d\log\Big(\frac{3}{\epsilon}\Big)\Bigg].
\]

\paragraph{Step 1: bound the $K$-term.}
Since $\lambda\ge 1$ and $n\ge 2$, $\log((\lambda+n)/\lambda)\le \log(1+n)\le 2\log n\le 2\ell$,
hence $K\le 1+500\ell\le 501\ell$ and thus
\[
\log\Big(\frac{2K}{\delta}\Big)\le \log\Big(\frac{1002\,\ell}{\delta}\Big)
\le 10\ell,
\]
using $\ell\ge 1$, $\log\ell\le \ell$, and $\log(1/\delta)\le \ell$.

\paragraph{Step 2: bound the $\epsilon$-term by $\ell+\tfrac12\log m$.}
Since $\bar\gamma=40$, we have $3/\epsilon = 120L$, so
\[
\log\Big(\frac{3}{\epsilon}\Big)=\log(120)+\log L
\le \log(240)+\log\gamma_{n-1}^{\delta/m}+\tfrac12\log(1+n/\lambda).
\]
Using $\lambda\ge 1$ and $n\ge 2$, $\tfrac12\log(1+n/\lambda)\le \tfrac12\log(1+n)\le \ell$.
Let $x:=\log(4m/\delta)$. By the definition of $\gamma_t^\alpha$ and $n\ge d$,
\[
 \gamma_{n-1}^{\delta/m}
=\sqrt d+2\sqrt{2x+d\log\Big(1+\frac{n-1}{\lambda d}\Big)}
\le \sqrt d+2\sqrt{2x+2d\ell} \leq 4\sqrt{d\ell+x}\,.
\]
So,
\[
\log\gamma_{n-1}^{\delta/m} \le \log 4 + \tfrac12\log(d\ell+x).
\]
Assume $m\ge 3\ell$ (this will be implied by $m\ge 2000\,d\,\ell$).
Then $\log(4m)\le m$ and $\log(1/\delta)\le \ell$, so $x=\log(4m/\delta)\le m+\ell\le 2m$.
Also $d\ell\le m$ (since $m\ge 2000\,d\,\ell$), hence $d\ell+x\le 3m$ and thus
\[
\log\gamma_{n-1}^{\delta/m} \le \log 4 + \tfrac12\log(3m) \le 3 + \tfrac12\log m.
\]
Therefore
\[
\log\Big(\frac{3}{\epsilon}\Big)\le \ell + 9 + \tfrac12\log m.
\]

\paragraph{Step 3: reduce to a log-inequality and solve it.}
Combining Steps 1--2,
\[
m_0(c,\delta',p)
\le 40\Big(10\ell + d(\ell+9+\tfrac12\log m)\Big)
\le 20d\log m + 800\,d\,\ell,
\]
where we used $d\ge 1$ and $\ell\ge 1$ to absorb constants.
Thus it suffices to ensure
\[
m \ge 20d\log m + 800\,d\,\ell.
\]
Apply \cref{lem:solve-log-ineq} with $A:=20d$ and $B:=800\,d\,\ell$ to get the sufficient condition
\[
m \ge 40d\log(20d) + 1600\,d\,\ell.
\]
Since $n\ge d$ and $\delta\le 1$, we have $\log(20d)\le \log(20n)\le 4\ell$, hence
$40d\log(20d)\le 160\,d\,\ell$.
Therefore $m\ge 2000\,d\,\ell$ implies $m\ge m_0(c,\delta',p)$, and the result follows from
\cref{lem:from-fixed-to-uniform}.
\end{proof}

\section{Upper bound on $\gamma_t^{\delta/n}$}
\label{app:gamma-bound}
\begin{lemma}[Upper bound on $\gamma_t^{\delta/n}$]\label{lem:gamma-bound}
Assume $\lambda\ge 1$, $n\ge \max(2,d)$, and $\delta\in(0,1]$. Let
$\ell:=\max\{1,\log(n/\delta)\}$. Then for all $t\in[n]$,
\[
\gamma_{t}^{\delta/n} \le 10\sqrt{d\,\ell}.
\]
\end{lemma}
\begin{proof}
By definition,
\[
\gamma_{t}^{\delta/n}
=
\sqrt d
+
2\sqrt{2\log(4n/\delta) + d\log\Big(1+\frac{t}{\lambda d}\Big)}.
\]
Since $t\le n$ and $\lambda\ge 1$, $\log(1+t/(\lambda d))\le \log(1+n)\le 2\log n\le 2\ell$.
Also,
\[
\log(4n/\delta)\le \log 4+\log(n/\delta)\le 3\ell.
\]
Therefore,
\[
\gamma_{t}^{\delta/n}
\le
\sqrt d+2\sqrt{6\ell+2d\ell}
\le 10\sqrt{d\,\ell},
\]
since $d\ge 1$ and $\ell\ge 1$.
\end{proof}

\section{Exceedance control for diagonal Gaussian martingale transforms}
\label{app:diag-gauss-exceedance}
\begin{corollary}[Scalar transforms with a common clock via Brownian embedding]\label{cor:diag-gauss-exceedance-bm}
Let $M_t=\sum_{s=0}^t D_s \xi_s$ be a diagonal martingale transform of a standard $m$-dimensional Gaussian noise,
where each $D_s$ is $\mathcal F_{s-1}$-measurable and scalar (the same across coordinates), and
$(\xi_s)_{s\ge 0}$ are i.i.d.\ $\mathcal N(0,I_m)$. Define the common clock
\[
A_t^2:=\sum_{s=0}^t D_s^2,\qquad t\ge 0.
\]
Fix $0<\tau < \tau'<\infty$, $c>0$, $0<p<p_0(c)$ and $\delta\in(0,1)$. If $A_t^2\in[\tau,\tau']$ almost surely
for all $t\in\{0,1,\dots,n-1\}$ and
\[
m \ge \frac{4}{p}\log\!\Big(\frac{K}{\delta}\Big),
\]
where $K$ is as in \cref{thm:bm-quantile-K} (choose any $h\in(0,h_\star]$ and set
$K=\lceil \log(\tau'/\tau)/h\rceil$), then with probability at least $1-\delta$,
\[
\min_{t\in\{0,1,\dots,n-1\}} \frac{1}{m}\sum_{j=1}^m
\ind*{\frac{M_{t,j}}{A_t}\ge c}
\;\ge\; p.
\]
\end{corollary}
\begin{proof}
By \cref{thm:diag-gauss-embed}, on an extension there exist independent standard Brownian motions
$W^1,\dots,W^m$ such that for all $t\in\{0,1,\dots,n-1\}$ and all $j\in[m]$,
$M_{t,j}=W^j(A_t^2)$. Hence
\[
\frac{1}{m}\sum_{j=1}^m
\ind*{\frac{M_{t,j}}{A_t}\ge c}
=
\frac{1}{m}\sum_{j=1}^m
\ind*{\frac{W^j(A_t^2)}{\sqrt{A_t^2}}\ge c}.
\]
Since $A_t^2\in[\tau,\tau']$ for all $t$, the time-uniform bound of \cref{thm:bm-quantile-K}
applies on $[\tau,\tau']$ and yields the claim. This scalar case is essentially the situation that arises
in our main proof.
\end{proof}

What if the diagonal transforms are not scalar multiples of the identity? In the general diagonal case,
the embedding gives $M_{t,j}=W^j(A_{t,j}^2)$ with potentially different clocks across coordinates.
Attempting to lower bound
\[
\frac{1}{m}\sum_{j=1}^m \ind*{\frac{W^j(A_{t,j}^2)}{\sqrt{A_{t,j}^2}}\ge c}
\]
uniformly in $t$ is much harder, because the times $A_{t,j}^2$ can vary with $j$ and with the history.
In the worst case one is led to control
\[
\inf_{\tau\le t_1,\dots,t_m\le \tau'} \frac{1}{m}\sum_{j=1}^m \ind*{W^j(t_j)\ge c\sqrt{t_j}},
\]
which is a substantially weaker object: even for independent Brownian motions, choosing different stopping
times for each coordinate can drive the average exceedance down as $\tau'$ grows.

For comparison, we can still obtain a direct bound when the diagonal coefficients are independent of the
Gaussian noise; this avoids the adaptive time-change issue but only yields a $\log n$ dependence.
Due to the independence assumption this result is not useful for our main proof and is included mainly for comparison purposes.

\begin{proposition}[Diagonal transforms with independent coefficients]\label{prop:diag-gauss-exceedance}
  Let $M_t=\sum_{s=0}^t D_s \xi_s$ be a diagonal martingale transform of a standard $m$-dimensional Gaussian noise,
where each $D_s=\diag(D_{s,1},\dots,D_{s,m})$ and the coefficient sequence $(D_s)_{s\ge 0}$ is independent of
$(\xi_s)_{s\ge 0}$, with $(\xi_s)_{s\ge 0}$ i.i.d.\ $\mathcal N(0,I_m)$.
Define the coordinate clocks
\[
A_{t,j}^2:=\sum_{s=0}^t D_{s,j}^2,\qquad t\ge 0,\ j\in[m],
\]
and assume $A_{t,j}>0$ for all $t\in\{0,1,\dots,n-1\}$ and $j\in[m]$. Fix
$c>0$, $\delta\in(0,1)$ and $0<p\le \frac12(1-\Phi(c))$. If
\[
m \ge \frac{4}{p}\log\!\Big(\frac{n}{\delta}\Big),
\]
then with probability at least $1-\delta$,
\[
\min_{t\in\{0,1,\dots,n-1\}} \frac{1}{m}\sum_{j=1}^m
\ind*{\frac{M_{t,j}}{A_{t,j}}\ge c}
\;\ge\; p.
\]
\end{proposition}
\begin{proof}
Fix $t\in\{0,1,\dots,n-1\}$. Let $D= (D_s)_{s\ge 0}$ denote the entire coefficient sequence.
Conditional on $D$, the variables
$M_{t,j}/A_{t,j}$ are independent standard normal across $j$, so each indicator has mean
$q:=1-\Phi(c)$.
By the multiplicative Chernoff bound,
\[
\mathbb P\Big(\frac{1}{m}\sum_{j=1}^m \ind*{\frac{M_{t,j}}{A_{t,j}}\ge c}<p\;\Big|\;D\Big)
\le \exp\!\Big\{-\frac{q m}{8}\Big\}
\le \exp\!\Big\{-\frac{p m}{4}\Big\},
\]
where we used $p\le q/2$.
Taking expectation and then a union bound over $t=0,\dots,n-1$ gives the claim under the stated condition on $m$.
\end{proof}

\paragraph{Comparison.}
If the clocks are identical across coordinates and we only know that the maximal clock value is
almost surely $O(n)$, then the proposition yields $m\gtrsim \log n$ via a union bound over $t$.
The Brownian corollary improves this to $m\gtrsim \log\log n$ once the clock is also bounded away from $0$
(so that $\log(\tau'/\tau)=O(\log n)$), without even using independence between the coefficients and the noise.

\section{Lower bound on ensemble size (proof of \cref{thm:lower-bound})}
\label{sec:ensemble-size-lb}

We prove that the size of the ensemble needs to scale at least linearly with the dimension $d$.

\lowerBoundThm*

The idea is that for certain arms, such as the unit ball, ensemble sampling will always play in the span of its initialisation---if this span does not capture $\theta_\star$, regret will be linear.

\begin{lemma}\label{lemma:es-lower-bound}
  Let $U =\spn\{\zeta^1, \dots, \zeta^m\}$ be the span of the initialisation of the ensemble. The actions $X_1, X_2, \dots$ chosen by ES satisfy
  \[
    (\forall t \geq 1)\qquad X_t \in U \quad \text{almost surely.}
  \]
\end{lemma}

\begin{proof}[Proof of \cref{lemma:es-lower-bound}]
  At initialisation, for all $j \in [m]$, $\theta^j_0$ is proportional to $\zeta^j$, and hence $\theta^j_0 \in U$.

  Now fix $t > 1$ and suppose that $X_1, \dots, X_{t-1} \in U$.
  Then both $\sum_{s=1}^{t-1} X_s Y_s \in U$ and $\sum_{s=1}^{t-1} X_s \xi_s^j \in U$.
  Moreover, $V_{t-1} = \lambda I + \sum_{s=1}^{t-1} X_s X_s\tran$ maps $U \mapsto U$.
  Hence, its inverse $V_{t-1}^{-1}$ also maps $U \mapsto U$ because the restriction of $V_{t-1}$ to $U$ is invertible.

  It follows that $\hat\theta_{t-1}\in U$ and for each $j\in [m]$, $V_{t-1}^{-1} S_{t-1}^j \in U$. Hence,
  \[
    \theta^j_{t-1} = \hat\theta_{t-1} + \bar\gamma \beta_{t-1}^\delta V_{t-1}^{-1} S_{t-1}^j \in U
  \]
  also holds.
  Since the action set is $\cX = \Bd$, $X_t = \theta_{t-1}^{J_t}/\norm{\theta_{t-1}^{J_t}}$ or $X_t = 0$ almost surely, and so
  \[
     X_t \in \spn\{\theta_{t-1}^1, \dots, \theta^m_{t-1}\} \subset U\quad \text{almost surely.}
  \]
  The claim follows by induction.
\end{proof}

Next, we will need the upcoming corollary to this claim to move from spans to regret.

\begin{lemma}\label{lem:deterministic-subspace}
  Let $U \subset \Rd$ be a deterministic linear subspace with $\dim(U) = k \leq d/2$ and let $V$ be a random vector distributed uniformly on $\Sd$. Then
  \[
    \P{\norm{\Pi_U V}^2 \leq 1/2} \geq 1/2\,.
  \]
\end{lemma}

\begin{proof}
  Without loss of generality assume that $U=\spn\{e_1, \dots, e_k\}$. Write $V = G/\norm{G}$ for a standard normal vector $G$ on $\Rd$. Then
  \[
    \norm{\Pi_U V}^2 = \frac{\norm{\Pi_U G}^2}{\norm{G}^2} = \frac{\sum_{i=1}^k G_i^2}{\sum_{i=1}^k G_i^2 + \sum_{j=k+1}^d G_j^2} =: \frac{A}{A+B}\,,
  \]
  where $A \sim \chi^2_k$, $B \sim \chi^2_{d-k}$ and $A$ and $B$ are independent. Write $B = B_1 + B_2$ where $B_1 \sim \chi^2_k$ and $B_2 \sim \chi^2_{d-2k}$ are independent. Then
  \[
    \P[\bigg]{\frac{A}{A+B} \leq \frac{1}{2}} = \P{A \leq B} \geq \P{A \leq B_1} = 1/2\,,
  \]
  where the final equality uses that $A$ and $B_1$ are independent and have the same law.
\end{proof}

\begin{corollary}\label{lem:worst-direction-exists}
  Let $U \subset \Rd$ be any random linear subspace of $\Rd$ with $\dim(U) \leq d/2$. Then there exists a deterministic $\theta_\star \in \Sd$ such that
  \[
    \P{\norm{\Pi_U \theta_\star}^2 \leq 1/2} \geq 1/2\,,
  \]
  where $\Pi_U$ denotes the orthogonal projection onto $U$.
\end{corollary}

\begin{proof}
  Let $V$ be a uniform random vector on $\Sd$ independent of $U$. Conditional on $U$, apply \cref{lem:deterministic-subspace} with $k = \dim(U) \leq d/2$ and take expectation to get that
  \[
    \P{\norm{\Pi_U V}^2 \leq 1/2} = \E_U[\P{\norm{\Pi_U V}^2 \leq 1/2 \mid U}] \geq 1/2\,.
  \]
  On the other hand, since we also have that
  \[
    \P{\norm{\Pi_U V}^2 \leq 1/2} = \E_V[\P{\norm{\Pi_U V}^2 \leq 1/2 \mid V}]\,.
  \]
  Therefore, there exists a $\theta_\star \in \Sd$ such that $\P{\norm{\Pi_U \theta_\star}^2 \leq 1/2} \geq 1/2$.
\end{proof}

\begin{proof}[Proof of \cref{thm:lower-bound}]
  Let $U =\spn\{\zeta^1, \dots, \zeta^m\}$ and let $\Pi_U$ denote the orthogonal projection onto $U$. Since $\dim U \leq m \leq d/2$, by \cref{lem:worst-direction-exists}, there exists a $\theta_\star \in \Sd$ such that $\P{\norm{\Pi_U \theta_\star}^2 \leq 1/2} \geq 1/2$. Take this $\theta_\star$ to be the instance parameter. On the aforementioned event, since $X_1, X_2, \dots \in U \cap \Bd$ almost surely (\cref{lemma:es-lower-bound}),
  \[
    R_n = n - \sum_{t=1}^n \langle X_t, \theta_\star \rangle \geq n(1 - \sup_{u \in U \cap \Bd} \langle u, \theta_\star \rangle) = n(1 - \norm{\Pi_U \theta_\star}) \geq (1-\sqrt{2}/2)n \geq n/4\,.\qedhere
  \]
\end{proof}

\printbibliography
\end{refsection}

\end{document}